\documentclass[11pt]{article}
\usepackage{fullpage}
\usepackage{graphicx} %
\usepackage{amsmath}
\usepackage{amsfonts}
\usepackage{amsthm}
\usepackage{mathtools}
\usepackage{booktabs}
\usepackage{makecell}
\usepackage[numbers]{natbib}

\usepackage{hyperref}
\usepackage{cleveref}

\newtheorem{theorem}{Theorem}[section]
\newtheorem{definition}[theorem]{Definition}
\newtheorem{lemma}[theorem]{Lemma}
\newtheorem{corollary}[theorem]{Corollary}

\newcommand*{\R}{{\mathbb{R}}}

\newcommand*{\Y}{\mathcal{Y}}
\newcommand*{\X}{\mathcal{X}}
\newcommand*{\C}{\mathcal{C}}
\newcommand*{\D}{{\mathcal{D}}}
\renewcommand*{\P}{\mathcal{P}}
\newcommand*{\T}{{\mathcal{T}}}
\newcommand{\dd}{\mathrm d}
\newcommand{\CAL}{\mathrm {CAL}}
\newcommand{\ind}{\mathbb I}

\newcommand*{\Dp}{{\mathcal{D}^p}}
\newcommand*{\yp}{{\mathbf{y}^p}}

\newcommand*{\x}{\mathbf{x}}
\newcommand*{\y}{\mathbf{y}}

\newcommand{\ECE}{\mathrm{ECE}}
\newcommand{\smCE}{\mathrm{smCE}}
\newcommand{\degCE}{\mathrm{degCE}}
\newcommand{\dCE}{\mathrm{dCE}}
\newcommand{\uCE}{\overline{\dCE}}
\newcommand{\lCE}{\underline{\dCE}}
\newcommand{\intCE}{\mathrm{intCE}}
\newcommand{\cfdl}{\mathrm{CFDL}}
\newcommand{\cdl}{\mathrm{CDL}}

\newcommand{\CE}{\mathrm{CE}}
\newcommand{\Cal}{\mathrm{Cal}}

\let\eps\epsilon
\let\phi\varphi

\DeclareMathOperator*{\E}{\mathbb{E}}
\DeclareMathOperator*{\ex}{\mathbb{E}}

\renewcommand{\Im}{\mathsf{Image}}
\DeclareMathOperator*{\argmax}{arg\,max}
\DeclareMathOperator*{\argmin}{arg\,min}
\DeclareMathOperator*{\minimize}{minimize}
\newcommand{\Ind}{\mathbf{1}}

\let\hat\widehat

\DeclareMathOperator{\sgn}{sign}

\newcommand{\ignore}[1]{}

\newcommand{\eat}[1]{}

\newcommand{\fr}[1]{\frac{1}{#1}}

\newcommand{\width}{\mathrm{width}}

\renewcommand{\v}{\mathbf{v}}
\renewcommand{\x}{\mathbf{x}}
\renewcommand{\y}{\mathbf{y}}

\newcommand{\zo}{\{0,1\}}
\newcommand{\p}{\mathbf{p}}

\DeclareMathOperator{\TV}{TV}
\DeclareMathOperator{\EMD}{EMD}

\newcommand{\mD}{\mathcal{D}}

\newcommand{\vstar}{{v^*}}
\newcommand{\recal}{\hat{p}}

\numberwithin{equation}{section}

\title{Calibration through the Lens of Indistinguishability\thanks{This is the full version of a survey that appears in the ACM SIGecom Exchanges \cite{GH-survey} .}}
\author{Parikshit Gopalan \\ Apple \and Lunjia Hu\thanks{Work done in part while Lunjia Hu was a Postdoctoral Fellow at Harvard University, supported by the
Simons Foundation Collaboration on the Theory of Algorithmic Fairness and the Harvard Center for Research on
Computation and Society (CRCS).} \\ Northeastern University}
\date{}

\begin{document}

\maketitle

\begin{abstract} 
Calibration is a classical notion from the forecasting literature which aims to address the question: how should predicted probabilities be interpreted? In a world where we only get to observe (discrete) outcomes, how should we evaluate a predictor that hypothesizes (continuous) probabilities over possible outcomes? The study of calibration has seen a surge of recent interest, given the ubiquity of probabilistic predictions in machine learning. This survey describes recent work on the foundational questions of how to define and measure calibration error, and what these measures mean for downstream decision makers who wish to use the predictions to make decisions. A unifying viewpoint that emerges is that of calibration as a form of indistinguishability, between the world hypothesized by the predictor and the real world (governed by nature or the Bayes optimal predictor). In this view, various calibration measures quantify the extent to which the two worlds can be told apart by certain classes of distinguishers or statistical measures.
\end{abstract}

\section{Introduction}
\label{sec:intro}

Prediction is arguably the ubiquitous computational task of our time. Every day, a remarkable amount of computational resources are invested in the prediction of various probabilities, whether it is a language model trying to answer a user's ambiguous query or a recommendation engine trying to predict which product/profile to show a user. These automated predictions affect nearly every aspect of our lives, be it social, medical or financial. What makes prediction different from more classical computational tasks (such as sorting numbers or computing max-flows) is that there is no well-defined notion of what constitutes correctness. 

To explore this issue in greater detail, let us consider the simplified setting of binary prediction, where nature is modeled as a joint distribution $\D^*$ over attributes $\x$ drawn from a domain $\X$ and labels $\y \in \Y$. In this article, we will mainly focus on the setting $\Y = \zo$ of Boolean labels.\footnote{We use boldface for random variables, thus $\x$ is random variable drawn from $\X$ whereas $x \in \X$ is a point in the domain.} We denote the marginal distribution over $\X$ by $\D^*_\X$, and $\Y$ by $\D^*_\Y$. A predictor is a function $p:\X \to [0,1]$. The {\em ground truth} in this setting is represented by the Bayes optimal predictor $p^*(x) = \ex[\y^*|\x = x]$. The obvious formulation  of correctness in prediction might be to learn $p^*$. The challenge is that we never see the values of $p^*$ itself, our only access to it is via the labels $\y^*$ which satisfy $\ex[\y^*|x] = p^*(x)$. So the obvious formulation of correctness, as finding $p$ which is close to $p^*$ under some suitable measure of distance, will not work. There are (at least) two different and complementary approaches: loss minimization and calibration.

\subsection{Loss minimization} 
In loss minimization, we choose a loss function $\ell: \zo \times [0,1] \to \R$, and a hypothesis class of predictors $\P = \{p:\X \to [0,1]\}$, and aim to find the predictor that minimizes \[ p = \argmin_{p' \in \C} \ex_{(\x, \y^*) \sim \D^*} [\ell(\y^*, p'(\x))]. \]
In essence, we use the labels $\y^*$ as a proxy for $p^*$, while $\ell$ plays the role of a distance measure. But it turns out that (for any proper loss) we indeed find the predictor in our family $\P$ that is closest to $p^*$. This is a consequence of the {\em bias variance} decomposition. Taking the example of squared loss $\ell(y,v) = (y - v)^2$, the decomposition tells us that for any predictor $p$, \footnote{It is easy to prove a similar statement about any {\em proper} loss, and a little harder to prove it about arbitrary losses. But the takeaway remains the same: by minimizing loss over a family $\P$, we find the best approximation to $p^*$ from $\P$ under a suitable notion of distance tailored to the loss. }

\begin{align*}
    \ex_{\D^*}[(\y^* - p(\x))^2] =  \underbrace{\ex_{\D^*}[(p(\x) - p^*(\x))^2]}_{\text{bias}} + \underbrace{\ex_{\D^*}[(\y^* - p^*(\x))^2]}_{\text{variance}}.
\end{align*}
Note that the variance is a property of $p^*$, independent of $p$.

Loss minimization is a simple yet immensely powerful paradigm that powers much of contemporary machine learning. But is it a satisfactory notion of correctness for prediction tasks? 
Here are some questions to consider:

\begin{itemize}
\item Imagine that a decision maker is using a predictor to make decisions that minimize their own loss function. This loss may differ from the one used to train the model, and might  differ across various decision makers. For instance, we could use forecasts about rain to decide whether or not to carry an umbrella, to decide whether to have a party outdoors or indoors, or whether to turn off the sprinklers. Each of these has its own loss function. Say our loss for carrying an umbrella when it does not rain is $0.1$, and for not carrying an umbrella when it rains is $0.9$. The optimal strategy here is to carry an umbrella on days when $p^*(x) \geq 0.1$. Now suppose that the predictor $p$ we have access to is not Bayes optimal. How do we make decisions using this predictor? Should we carry an umbrella whenever $p(x) \geq 0.1$, just like with the Bayes optimal predictor, or should we make decisions differently? 

\item We know that the squared loss decomposes into bias and variance, but we have no way of knowing how large each of these are. If we suffer large squared loss, it could because nature is inherently random (e.g $p^*$ is often close to $1/2$), or because nature is deterministic but sufficiently complex that it {\em looks random} to our hypothesis class $\P$. Loss minimization does not distinguish between these scenarios.  

\item Suppose we wish to predict the probability of rain tomorrow, and the model $p$ found by minimizing squared loss gives a $60 \%$ chance of rain. How should we interpret this prediction? Is it possible that although $p$ minimizes expected error globally, it is not particularly good at prediction for certain types of days (like days in September)? 
Concerns like these arise naturally in the context of fair predictions for subgroups (see the discussion on multicalibration in Section \ref{sec:limits}). 

\end{itemize}
The question of what a prediction really guarantees naturally leads us to calibration.

\subsection{Calibration} 
 Calibration is a notion of correctness that focuses on ensuring that predicted probabilities align with actual outcomes. Intuitively, on days when a calibrated predictor predicts a $60\%$ chance of rain, it rains $60\%$ of the time. Formally, we can define perfect calibration as follows:

\begin{definition}
\label{def:perfect-cal}
    The predictor $p: \X \to [0,1]$ is \emph{perfectly calibrated} under the distribution $\D^*$ if for every $v \in \Im(p)$, it holds that $\ex[\y^*|p(\x) = v]  = v$.
\end{definition}

A key property of calibration is that it simplifies downstream decision making. For instance, let us return to the problem of using the forecast about rain to decide whether or not to carry an umbrella, where our loss for carrying an umbrella when it does not rain is $0.1$, and for not carrying an umbrella when it rains is $0.9$. Now suppose that the predictor $p$ we have access to is not Bayes optimal, but it is calibrated. If we are basing decisions solely on $p$, then the optimal strategy is still to carry an umbrella on days when $p(x) \geq 0.1$. The expected loss we would suffer is exactly what we would have suffered if our predictor were Bayes optimal.

This naturally motivates  an alternate view of calibration as a notion of correctness for predictors based on indistinguishability from the Bayes optimal, which will be an important theme in this survey. This view is inspired by the outcome indistinguishability framework of \cite{OI}.\footnote{That work does not consider calibration {\em per se}, it instead considers more general notions such as multicalibration from \cite{hkrr2018}. In the context of calibration, it is plausible that this indistinguishability viewpoint predates it, though we have not found a reference. }

To every predictor $p: \X \to [0,1]$, we can associate a distribution $\Dp$ on pairs $(\x, \y^p)$ where the marginal on $\x$ is $\D^*_\X$ and where $\ex[\yp|\x] = p(\x)$. The Bayes optimal predictor for $\Dp$ is $p$. Perfect Calibration requires that the joint distributions $(p(\x), \yp)$ and $(p(\x), \y^*)$ be identical.

\begin{lemma}[Perfect calibration as indistinguishability]
\label{lem:perfect-cal}
    The predictor $p: \X \to [0,1]$ is perfectly calibrated under the distribution $\D^*$ iff the joint distributions $J^* = ((p(\x),\y^*))$ and $J^p = (p(\x), \y^p)$ on $[0,1] \times \zo$  are identical.
\end{lemma}

Let us see why this is true. Since the marginal distribution of $\x$ is the same in both cases, the distribution of $p(\x)$ is also the same. In essence, we require that the distributions $\y^*|p(\x)$ and $\yp|p(\x)$ be identical. Since the latter is the Bernoulli distribution with parameter $p(\x)$, we require the same for $\y^*|p(\x)$, which is the standard definition. This guarantee conditional on each prediction is the key strength of calibration as a prediction guarantee.\footnote{Of course, there might be a different calibrated predictor that only puts the chances of rain at $30 \%$. There is no contradiction because the level sets of the predictors over which we average are different in the two cases. }

The indistinguishability property asserts that $p(\x)$ be a plausible explanation for the observations $\y^*$ given $\x$, in that the conditional distribution of $\y^*|p(\x)$ is consistent with the hypothesis that $p(\x)$ is the Bayes optimal predictor.
This indistinguishability property is desirable in machine learning, where we often try to model complicated processes (like the likelihood of a medical condition) and are unlikely to find the true Bayes optimal.
Calibrated predictors are considered more trustworthy, whereas a predictor that is not calibrated will fail some basic tests: the probability of the label being $1$ conditioned on $p(\x) = v$ is not $v$. 

In the bigger picture, the notion of indistinguishability has played a central role in several disciplines within theoretical computer science, cryptography  and pseudorandomness to name just a couple, indeed its roots go back to the Turing test. Viewing calibration as a form of indistinguishability lets us draw on ideas from those areas when we seek to define approximate calibration or generalize our notions beyond the binary classification setting.

\subsection{From perfect to approximate calibration}
\label{sec:p-to-a}

Perfect calibration is a clean abstraction, but predictors trained and used for prediction tasks in the real world are seldom perfectly calibrated. For calibration to be a useful notion, we need to define what it means for a predictor to be approximately (but not perfectly) calibrated, and we need efficient methods to measure calibration error. How to do this in a principled manner is the main focus of this article.

There are many desiderata that one might hope a notion of approximate calibration satisfies:

\begin{enumerate}
    \item It should preserve the desirable properties of calibration, such as indistinguishability and simple downstream decision making, in some approximate sense.
    \item It should be efficient to measure the calibration error of a given predictor, just from black box access to samples of the form $(p(\x), \y^*)$, both in terms of sample complexity and computational complexity. In an online setting (to be defined shortly), we might wish for our notion to have low-regret algorithms. 
    \item The notion should be robust to small perturbations in the predictor. A tiny change to a calibrated predictor should not result in a predictor with huge calibration error. For instance, changing the days forecast from $60\%$ to $59.999\%$ should not result in wild swings in the calibration error.\footnote{This is especially desirable from a machine learning perspective, where the lower order bits of prediction are considered insignificant and typically disregarded in low-precision arithmetic.} 
    \item The notion should extend beyond binary classification, to multiclass labeling and regression, while maintaining properties like efficiency. 
\end{enumerate}

Achieving all of these properties is not easy. The classical notion of calibration error, which is the expected calibration error or $\ECE$, only satisfies property (1) above; we will discuss this in more detail in Section \ref{sec:cal-error}. An active line of recent research has yielded a rich theory of approximate notions of calibration, together with algorithms for computing them efficiently in various models. Yet, to date, there is no single notion that satisfies all four desiderata mentioned above! 

Perhaps this is too much to hope for, since some of these desiderata (eg. robustness and low-regret algorithms) arise from different motivating scenarios. But a clear takeaway from this body of research is  that approximate calibration is surprisingly challenging to define and measure.  The key technicality in defining approximate calibration error comes from conditioning. Every definition of calibration involves some form of conditioning on predictions. While this conditioning is simple for perfectly calibrated predictors, it is far trickier for predictors that are not perfectly calibrated, since predictions are real-valued. 

In this survey, we will highlight how the {\em indistinguishabilty} viewpoint on calibration guides us in formulating what approximate calibration should mean. At a high level, there are two approaches to this task: 
\begin{itemize}
\item {\bf Limit the set of  distinguishers :} Rather than require $J^*$ and $J^p$ be identical, we ask that they look similar to a family $W$ of distinguishers. The calibration error is measured by the maximum distinguishing advantage achieved over all distinguishers in $W$. This approach is directly inspired by cryptography and pseudorandomness.
\item {\bf Use a divergence/distance on distributions:} Since $J^*$ and $J^p$ are both distributions on the domain $[0,1] \times \zo$, we can use distance measures/divergences on probability distributions (e.g., total variation, earthmover) to measure the distance between them, and use this as our measure of the calibration error. As we will see, this view relates to a quantification of the economic value of calibration
from the perspective of downstream decision making.
\end{itemize}

These approaches lead to a number of calibration error measures that we will explore in more detail in this article, and which have many advantages over $\ECE$ and other traditional calibration measures. We will analyze smooth calibration error \cite{kakadeF08}, which satisfies properties (1-3) but not (4). It also corresponds to an intuitive notion of approximate calibration, where the predictor is close to some perfectly calibrated predictor in earthmover distance.

From the computational standpoint, the natural model in which to study calibration has been the online setting, where we measure the regret or calibration error of our prediction strategy over $T$ time steps.\footnote{Note that the computational task of learning a calibrated predictor admits trivial solutions in the offline model; for instance, one can always predict the expectation of the label.}
The classic work of \cite{FosterV98} showed that sublinear calibration error, as measured by $\ECE$ is possible. The regret rate achieved in their work is $O(T^{2/3})$.\footnote{The regret rate is $T$ times the calibration error on the uniform distribution over the $T$ time steps.}
It is known that regret rates of $O(\sqrt{T})$ or even $\widetilde O(\sqrt{T})$ are not possible for $\ECE$ \cite{sidestep}, and figuring out the optimal regret achievable is an active area of research (see, e.g., \cite{breaking}). However, 
new notions of calibration, which we will discuss in this survey, actually admit prediction strategies that achieve $O(\sqrt{T})$ or $\widetilde O(\sqrt{T})$ regret rates \cite{dist-online,elementary,cdl}.

\subsection{Limitations and generalizations} 
\label{sec:limits}

Calibration is clearly a  desirable property for a predictor, but it has limitations, and cannot be considered as a standalone notion of goodness for a predictor. We ideally want predictors to have both good calibration and other properties like small expected loss. We discuss these limitations below, and use this as motivation to introduce the stronger notion of multicalibration \cite{hkrr2018}, and discuss how it addresses these limitations.

\paragraph*{Calibration does not guarantee utility} There are many predictors that will satisfy calibration, and we would not consider all of them to be equally informative or good. For instance, the {\em average} predictor $\bar{p}$ that always predicts the average label $\ex_{\D^*}[\y^*]$ is perfectly calibrated, as is the Bayes optimal $p^*$. Any reasonable loss function would distinguish between these predictors, but calibration (by itself) does not.

\paragraph*{Calibration gives guarantees on average over the entire population} In some applications, this might not be good enough. For instance, suppose we train a predictor to predict the risk of a certain risk of disease for a patient. On examining the data, we find that although the predictor is calibrated over the general population, it is miscalibrated for patients with a certain medical history, who are a small fraction of the dataset (so this does not affect the overall calibration error too much). We would not trust such a predictor to make decisions for those patients. 

\paragraph*{Multicalibration}
Multicalibration, introduced in \cite{hkrr2018}, is a strengthening of calibration. It requires that our predictions are calibrated, even when conditioned on membership in a rich collection of demographic subgroups $\C \subseteq 2^\X$. Which subgroups to consider is an important consideration, which is dictated by the data and computational resources available to the predictor. We refer the reader to \cite{hkrr2018} for more details. 

Although calibration by itself does not guarantee good loss minimization,  multicalibration with respect to rich class of subgroups $\C$ does imply strong loss minimization. This was the key insight in the work of \cite{omni} which introduced the notion of omniprediction. Omniprediction asks for a predictor that is as good as benchmark class $\C$ not just for a single loss function, but for any loss from a large family of loss functions. \cite{omni} shows a surprising connection between omniprediction with respect to a benchmark class $\C$ and multicalibration with respect to $\C$. 
 
From the indistinguishability perspective, \cite{OI} showed that multicalibration is equivalent to indistinguishability of the distributions $(c(\x), p(\x), \y^*)$ and $(c(\x),\allowbreak p(\x), \y^p)$ for all $c: \X \to \zo$ that lie in some family $\C$ of functions. Beyond its original motivation in multigroup fairness, multicalibration has proved to be tremendously powerful, finding applications to omniprediction \cite{omni},  domain adaptation \cite{kim2022universal}, pseudorandomness \cite{DworkLLT}, and computational complexity \cite{CDV}.

\paragraph*{Organization of this survey}

In Section \ref{sec:ece}, we consider expected calibration error (ECE) and explore its weaknesses. In Section \ref{sec:cal-error}, we introduce weighted  calibration measures which capture the notion of indistinguishability to limited classes of distinguishers. This unifies several different notions of approximate calibration in the literature. In Section~\ref{sec:decision}, we describe Calibration decision loss, which looks at calibration from an economics perspective, through the eyes of a downstream decision maker who wants to use the predictions of a predictor to optimize their utility. We review the active area of research on online calibration in \Cref{sec:online}.  Given the number of calibration notions that we will encounter,  a natural question is whether there is some ground truth notion against which we can compare these different notions.  In Section \ref{sec:dtc}, we define the distance to calibration, which proposes a ground-truth notion of what approximate calibration ought to mean, and show how smooth calibration shows up naturally in this setting.

\section{Expected Calibration Error}
\label{sec:ece}

We start with what is arguably the most popular metric for measuring calibration error: the expected calibration error or $\ECE$. We examine some of its shortcomings, which will guide us in formulating other notions of approximate calibration.

\begin{definition}
    The \emph{expected calibration error} of a predictor $p$ under $\D^*$ is defined as $\ECE(p, \D^*) = \ex\left|\ex[\y^*|p(\x)]  - p(\x)]\right|$.
\end{definition}

Some notes on the definition of ECE:
\begin{itemize}
\item While perfect calibration requires $\ex[\y^*|p(\x)] = p(\x)$, ECE allows for some slack in the equality, and measures the average deviation over all $p$. 
\item We have defined ECE as measuring the absolute deviation between $\ex[\y^*|p(\x)]$ and $p(\x)$. We could instead have used the square or the $q^{th}$ power for $q \geq 1$ and defined $\ECE_q(p, \D^*) = \ex[|\ex[\y^*|p(\x)] - p(\x)|^q]^{1/q}$. By the convexity of $t^q$, $\ECE_q$ is an increasing function of $q$. 
\end{itemize}

For a better understanding of ECE, we look at two alternative characterizations. The first characterizes it in terms of the maximum inner product with a distinguisher $b$ which is a bounded function on $[0,1]$.

\begin{lemma}
\label{lem:ece-bounded}
    Let $B= \{b:\zo \to [-1,1]\}$ be the family of all bounded functions. Then $\ECE(p, \D^*) = \max_{b \in B}\E_{J^*}[b(p(\x))(\y^* - p(\x))]$.
\end{lemma}

\begin{proof}
    For each $v \in \Im(p)$, define the function 
    \[ b^*(v) = \sgn(\E[\y^*|p(\x) = v] - v).\]
    Then observe that
    \begin{align*} 
    \E[b(p(\x))(\y^* - p(\x))] & \leq \E[b^*(p(\x))(\y^* - p(\x))]\\ &=\E[\sgn(\E[\y^*|p(\x)] - p(\x))(\E[\y^*|p(\x)] - p(\x))]\\
    &= \E[\left|\E[\y^*|p(\x)] - p(\x)\right|]\\
    &= \ECE(p, \D^*),
    \end{align*}
which proves the equality.
\end{proof}

 For two distributions $\D_1, \D_2$ on a domain $\X$, we define
\[ \TV(\D_1, \D_2) = \max_{S \subseteq \X} |\D_1(S) - \D_2(S)|.
\footnote{When the space $\X$ is infinite, we must restrict $S$ to be measureable, but we will ignore this and other such subtleties.} \]

We state the second characterization in terms of total variation distance.
\begin{lemma}
\label{lem:ece-tv}
    We have $\ECE(p, \D^*) = \TV(J^*, J^p)$.
\end{lemma}
\begin{proof}
By the definitions of $J^*$ and $J^p$, we have
    \begin{align}
    J^*(S) & = \E_{\x}[p^*(\x)\mathbb I[(p(\x),1)\in S] + ( 1 - p^*(\x))\mathbb I[(p(\x),0)\in S]], \label{eq:j-s}\\
    J^p(S) & = \E_{\x}[p(\x)\mathbb I[(p(\x),1)\in S] + ( 1 - p(\x))\mathbb I[(p(\x),0)\in S]].\label{eq:jp-s}
    \end{align}
For every $v\in [0,1]$, define
\begin{equation}
\label{eq:bs}
b_S(v):= \mathbb I[(v,1)\in S] - \mathbb I[(v,0)\in S]\in [-1,1].
\end{equation}
Taking the difference between \eqref{eq:j-s} and \eqref{eq:jp-s}, we get
\begin{align}
J^*(S) - J^p(S) & = \E_\x[(p^*(\x) - p(\x))(\mathbb I[(p(\x),1)\in S] - \mathbb I[(p(\x),0)\in S])]\notag \\
& = \E[(\y^* - p(\x))b_S(p(\x))].\label{eq:j-diff}
\end{align}
Combining this with \Cref{lem:ece-bounded} we get $|J^*(S) - J^p(S)| \le \ECE(p,\D^*)$, and thus $\TV(J^*,J^p) \le \ECE(p,\D^*)$.

Now we choose $S$ such that for every $v\in [0,1]$,
\begin{align*}
    & (v,1)\in S \text{ and } (v,0)\notin S \quad \text{if $\sgn(\E[\y^*|p(\x) = v] - v) = 1$};\\
    & (v,1)\notin S \text{ and } (v,0)\in S \quad \text{if $\sgn(\E[\y^*|p(\x) = v] - v) = -1$}.
\end{align*}
Plugging this into \eqref{eq:bs}, we get $b_S(v) = \sgn(\E[\y^*|p(\x) = v] - v)$. By \eqref{eq:j-diff} and the proof of \Cref{lem:ece-bounded},
\[
J^*(S) - J^p(S) = \E[(\y^* - p(\x))b_S(p(\x))] = \ECE(p,\D^*).
\]
This implies $\TV(J^*,J^p) \ge \ECE(p,\D^*)$.
\end{proof}

\paragraph*{The trouble with $\ECE$}

At first glance, $\ECE$ seems to be a reasonable measure of calibration error. However there are (at least) a couple of problems with it: it is hard to efficiently estimate (even in the binary classification setting), and it is very discontinuous. Thus it fails desiderata (2-4).

The computational difficulty stems from Lemma \ref{lem:ece-bounded}. Estimating the $\ECE$ is equivalent to finding the best witness $b \in B$. This is essentially a learning problem over a class with infinite VC dimension. Indeed, one can show that sample complexity of estimating $\ECE$ can be as large as $\Omega(\sqrt{|\X|})$. Ideally, we would like to complexity to be independent of the domain size, and depending only on the desired estimation error. 

The continuity problems are hinted at by Lemma \ref{lem:ece-tv}. While total variation distance is a good distance measure for distributions over discrete domains, it is not ideal for continuous domains. And our setting involving distributions over predictions in $[0,1]$ is inherently continuous. As the next example illustrates, $\ECE$ turns out to be highly discontinuous in the predictions of our predictor. 

\begin{itemize}
\item Let $\D_2$ be the uniform distribution a two point space $\{(a,0),(b,1)\}$, where $a$ is always labeled $0$ and $b$ is labeled $1$. 
\item Consider the predictor $p_0$ which predicts $1/2$ for both $a$ and $b$. It is perfectly calibrated, hence $\ECE(p_0) = 0$. 
\item For $\eps > 0$, define the predictor $p_\eps$ where
$p_\eps(a) = 1/2 - \eps,  p_\eps(b) = 1/2 + \eps$.
It is easy to verify that $\ECE(p_\eps) = 1/2 - \eps$. 
\end{itemize}
Think of $\eps$ being infinitesimally small but positive, so that $p_\eps$ is extremely close to $p_0$. Intuitively, $p_\eps$ is very close to being perfectly calibrated, it only requires a small perturbation of the lower order bits.  Yet, the $\ECE$ is close to $1/2$ for $p_\eps$, whereas it is $0$ for $p_0$. 

There are many ad-hoc fixes in practice that aim to get around these difficulties. For instance, bucketed $\ECE$ divides the interval $[0,1]$ into $b$ equal sized buckets, rounds the predictions in  each bucket (say to the midpoint) and then measures the $\ECE$ of the discretized predictor. But \cite{BlasiokGHN23} observe that this results in a bucketed $\ECE$ which oscillates between $0$ and $1/2 - \eps$ depending on whether the number of buckets is odd or even! 

Are our issues with $\ECE$ small technicalities or symptoms of a bigger problem? We believe it is the latter. Assume you are training a predictive model, and you measure its $\ECE$ and find it to be large. Is this something you should worry about? Is your model truly miscalibrated (whatever that means)? Or is there an infinitesimal perturbation of its predictions that will make it perfectly calibrated? In general, there are sound reasons to prefer metrics that are reasonably smooth. 
It is also important for estimation to be efficient in terms of both samples and computation, which is not the case for $\ECE$.

\section{Weighted calibration error}
\label{sec:cal-error}

In this section, we will explore notions of approximate calibration that only require that $J^*$ and $J^p$ look similar to a family $W$ of distinguishers or weight functions. This results in a general template called weighted calibration, which is parametrized by the family $W$. Instantiating this notion with the family of bounded Lipschitz functions, we  derive the notion of smooth calibration \cite{kakadeF08}. We briefly describe some other notions of calibration from the literature that can be viewed as instantiations of this template. 

\subsection{Weighted calibration}
Weighted calibration error \cite{GopalanKSZ22} captures the extent to which a collection of distinguishing functions are able to distinguish $J^*$ from $J^p$. 
Since $J^*$ and $J^p$ are both distributions over $[0,1] \times \zo$, we consider distinguishing functions $f:[0,1] \times \zo \to [-1,1]$. Since the second argument to $f$ is Boolean, we can write  $f(v,y) = w(v)y + u(v)$. Hence, 
 \begin{align} 
 \E_{J^*}[f(\v, \y^*)] - \E_{J^p}[f(\v,\y^p)]  &= \E_{J^*}[w(\v)\y^*] - \E_{J^p}[w(\v)\y^p]\notag = \E_{J^*}[w(\v)\y^*] - \E_{J^p}[w(\v)\v]\notag\\
 & = \E_{J^*}[w(\v)(\y^* -\v)].\label{eq:diff-exp}
\end{align}
where the first and third equalities hold because $\v$ is identically distributed under $J^*$ and $J^p$, and the second is because $\E[\y^p|\v] = \v$. This tells us that we can limit ourselves to distinguishers of the form $f(v,y) = w(v)y$, and the distinguishing advantage can be thought of as an expectation under the single distribution $J^*$ (Equation \eqref{eq:diff-exp}). This leads to the following definition from \cite{GopalanKSZ22}.

\begin{definition}[Weighted calibration error \cite{GopalanKSZ22}]
    Let $W= \{w:[0,1] \to [-1,1]\}$ be a family of weight functions. The \emph{$W$-weighted calibration error} of the predictor $p:\X \to [0,1]$ is defined as
    \[ \CE_W(p, \D^*) = \max_{w \in W}\left|\ex_{\D^*}[w(p(\x))(\y^* - p(\x))]\right|. \]
\end{definition}

The definition of weighted calibration error suggests a natural computational problem: the problem of calibration auditing for a weight family $W$. This is the computational problem of deciding whether $\CE_W(p, \D^*)$  is $0$ or exceeds $\alpha$, given access to random samples $(p(\x), \y^*)$ from $\D^*$. This problem turns out to be closely related to agnostic learning for the class $W$, as shown by \cite{GopalanHR24}. 

If we instantiate weighted calibration with  $W = B$ where $B$ is the set of all bounded functions introduced in \ref{lem:ece-bounded}, we recover $\ECE$. But this also illustrates why $\ECE$ is hard to compute efficiently: the set $B$ has infinite VC dimension, hence it cannot be learnt efficiently.

Note that we could have defined the weighted calibration error $\CE_W$ as a function of $J^*$, the joint distribution of $(p(\x), \y^*)$, rather than the pair $(p, \D^*)$. We prefer mentioning $p$ explicitly for clarity, but it is important to note that $\CE_W$ only depends on $J^*$. Indeed, most common measures of calibration error and loss only depend on the distribution of $J^*$. For instance, the cross-entropy loss and square loss only depend on how labels and predictions are jointly distributed, not on whether we are labeling images or tabular data; if we predict $p(x) = 0.7$ and the label is $1$, that fixes the loss suffered at $x$, regardless of the features $x$.

\subsection{Smooth calibration}

Smooth calibration, introduced by \cite{kakadeF08} is an instantiation of weighted calibration that restricts the class of weight functions to Lipschitz continuous functions. This ensures that small perturbations of the predictor do not result in large changes in the calibration error.

\begin{definition}
\label{def:smooth}
    Let $L = \{l:[0,1] \to [-1,1]\}$ denote the subset of $1$-Lipschitz functions from $B$. Define the \emph{smooth calibration error} of the predictor $p$ under the distribution $\D^*$ as
    $\smCE(p, \D^*) = \CE_L(p, \D^*)$.
\end{definition}

By only allowing Lipschitz weight functions, Smooth calibration ensures that the calibration error does not change dramatically under small perturbations of the predictor.\footnote{Note that  Lemma \ref{lem:ece-bounded} tells us that there exists a bounded function $b_\eps$ that explains the high $\ECE$ for $p_\eps$, specifically,
$b_\eps(v) = \sgn(v - 1/2)$. This function is discontinuous near $1/2$, which causes the extreme sensitivity to perturbations. }
Given predictors $p_1, p_2: \X \to [0,1]$ and a distribution $\D^*$ on $\X$, let the expected $\ell_1$ distance between them be
\[ d(p_1, p_2) = \E_{\D^*}[|p_1(\x) - p_2(\x)|].\]
Smooth calibration error is Lipschitz in this distance.
\begin{lemma}
    For any weight family $W \subseteq L$, $\CE_W(p, \D^*)$ is $2$-Lipschitz in $d$.
\end{lemma}
\begin{proof}
    It suffices to prove that for every pair of predictors $p_1,p_2$ and every $w\in W$, it holds that
    \[
    |\E\nolimits_{\D^*}[w(p_1(\x))(\y^* - p_1(\x))] - \E\nolimits_{\D^*}[w(p_2(\x))(\y^* - p_2(\x))]| \le 2d(p_1,p_2).
    \]
    To this end, we consider the intermediate quantity $\E\nolimits_{\D^*}[w(p_1(\x))(\y^* - p_2(\x))]$. We have
    \begin{align*}
        & |\E\nolimits_{\D^*}[w(p_1(\x))(\y^* - p_1(\x))] - \E\nolimits_{\D^*}[w(p_1(\x))(\y^* - p_2(\x))]|\\
        = {} & |\E[w(p_1(\x))(p_2(\x) - p_1(\x))]|\\
        \le {} & d(p_1,p_2) \tag{because $|w(p_1(\x))| \le 1$},
    \end{align*}
and
    \begin{align*}
        & |\E\nolimits_{\D^*}[w(p_1(\x))(\y^* - p_2(\x))] - \E\nolimits_{\D^*}[w(p_2(\x))(\y^* - p_2(\x))]|\\
        \le {} & \E|w(p_1(\x)) - w(p_2(\x))| \tag{because $|\y^* - p_2(\x)| \le 1$}\\
        \le {} & d(p_1,p_2). \tag{because $w$ is $1$-Lipschitz}
    \end{align*}
Using the triangle inequality to combine the two inequalities above completes the proof.
\end{proof}

Returning to the example above with $p_0$ and $p_\eps$, restricting to Lipschitz distinguishers means that smooth calibration considers $p_\eps$ to also be well calibrated, since its smooth calibration error is $O(\eps)$.

An alternate view of smooth calibration is in terms of earthmover distance between $J^*$ and $J^p$.  Consider the $\ell_1$ metric on $[0,1] \times \zo$ where $\ell_1((v,y), (v',y')) = |v - v'| + |y - y'|$. For two distributions $J, J'$ on $[0,1] \times \zo$, we denote the earthmover distance between two distributions under the $\ell_1$ metric as $\EMD(J, J')$. Smooth calibration captures the earth-mover distance between $J^*$ and $J^p$.

\begin{lemma}
\label{lem:smce-em}
    We have $\EMD(J^*, J^p)/2 \leq \smCE(p, \D^*) \leq \EMD(J^*, J^p)$.
\end{lemma}
This lemma should be contrasted with Lemma \ref{lem:ece-tv}, which characterizes $\ECE$ in terms of the total variation distance.
\begin{proof}[Proof of \Cref{lem:smce-em}]
Let $\v\in [0,1],\y^*,\y^p\in \{0,1\}$ be random variables such that the joint distribution of $(\v,\y^*)$ is $J^*$ and the joint distribution of $(\v,\y^p)$ is $J^p$.
    By the Kantorovich-Rubinstein Duality,
    \[
    \EMD(J^*, J^p) = \sup_{f\in F}(\E[f(\v,\y^*)] - \E[f(\v,\y^p)]),
    \]
where $F$ consists of all functions $f:[0,1]\times \{0,1\}\to \R$ that are $1$-Lipschitz in the $\ell_1$ metric $\ell_1((v,y), (v',y')) = |v - v'| + |y - y'|$.
    As in \eqref{eq:diff-exp}, we can write $f(v,y)$ as $w(v) y + u(v)$ using functions $w,u:[0,1]\to \R$, where $w(v) = f(v,1) - f(v,0)\in [-1,1]$ is a $2$-Lipschitz function of $v$. By \eqref{eq:diff-exp},
    \[
    \E[f(\v,\y^*)] - \E[f(\v,\y^p)] = \E[w(\v)(\y^* - \v)] \le 2\, \smCE(p,\D^*).
    \]
This proves $\EMD(J^*, J^p)/2 \leq \smCE(p, \D^*)$.

Now we take an arbitrary $1$-Lipschitz function $w:[0,1]\to [-1,1]$ and define $f(v,y):= w(v)y\in [-1,1]$ for every $v\in [0,1]$ and $y\in \{0,1\}$. By \eqref{eq:diff-exp},
\[
|\E[w(\v)(\y^* - \v)]| = |\E[f(\v,\y^*)] - \E[f(\v,\y^p)]| \le \EMD(J^*, J^p),
\]
where we used the fact that $f$ is a $1$-Lipschitz function of $(v,y)\in [0,1]\times \{0,1\}$ in the $\ell_1$ metric. This proves $\smCE(p, \D^*) \leq \EMD(J^*, J^p)$.
\end{proof}

We have defined smooth calibration error in terms of the family of $1$-Lipschitz distinguishers. But since an $L$-Lipschitz function for $L >1$ can be made $1$-Lipschitz by rescaling the range by $L$, the calibration error can only increase by $L$ even if we allow $L$-Lipschitz distinguishers.

\subsection{Other notions of weighted calibration}

We have seen two notions of weighted calibration so far: $\ECE$ and $\smCE$. Several other calibration metrics that have been considered in the literature can be naturally viewed as instances of weighed calibration. We list some of them below.

\begin{itemize}
    \item Low-degree calibration \cite{GopalanKSZ22} corresponds to the case where $W = P_d$ consists of degree $d$ polynomials. This class is fairly Lipschitz (since polynomials have bounded derivatives. The main attraction of this notion is that it is efficient to computer, even in the multiclass setting.
    \item In Kernel calibration \cite{KSJ18,BlasiokGHN23} the family of weight functions lies in a Reproducing Kernel Hilbert Space. There are many choices of kernel possible, such as the Laplace kernel, the Gaussian kernel or the polynomial kernel, each of these results in distinct calibration measures with their own properties.
\end{itemize}

\section{Calibration Error for Decision Making}
\label{sec:decision}

\eat{

We have discussed the notion of perfect calibration: equality between the distributions $J^* = (p(\x),\y^*)$ and $J^p = (p(\x), \y^p)$ (\Cref{def:perfect-cal}). We presented natural measures of the calibration error 
based on the advantage of distinguishing $J^*$ from $J^p$ using a family of distinguishers, or the distance between $J^*$ and $J^p$ according to specific metrics. However, there are natural measures of similarity between probability distributions, such as the KL divergence, that cannot be expressed as a distance metric or a distinguishing advantage. }

In this section, we will explore a second approach to relaxing the definition of perfect calibration, where rather than asking $J^*$ and $J^p$ be identical, we require them to be close when measured under a suitable divergence. This leads to another important measure of the calibration error, the Calibration Decision Loss (CDL), introduced recently by Hu and Wu \cite{cdl}. Underlying the notion of CDL is a concrete and natural quantification of the economic value of calibration from the perspective of downstream decision making.

We define the notion of CDL in \Cref{sec:cdl} and discuss its alternative formulation using Bregman divergences between $J^*$ and $J^p$ in \Cref{sec:div}. A key tool we use to prove this Bregman divergence formulation is a classic characterization of \emph{proper scoring rules} \cite{mccarthy,savage,gneiting}.

\subsection{Calibration Decision Loss}
\label{sec:cdl}

The definition of the Calibration Decision Loss comes naturally when we look at calibration through an economic lens, from the perspective of downstream decision makers.
What does calibration mean to a person who uses the predictions (e.g.\ chance of rain) to make downstream decisions (e.g.\ take an umbrella or not)?
We will show that a calibrated predictor provides a concrete \emph{trustworthiness} guarantee to \emph{every} payoff-maximizing downstream decision maker (\Cref{thm:cal-opt}). This observation gives not only a characterization of perfect calibration, but also a natural way of quantifying the calibration error of a miscalibrated predictor, 
using the payoff loss caused by trusting the (miscalibrated) predictor in downstream decision making. This way of quantifying the calibration error leads exactly to Calibration Decision Loss (\Cref{def:cdl}). 

We start by formally defining decision tasks.
A decision task $\T$ has two components: an action space $A$ and a payoff function $u:A\times \{0,1\}\to \R$. Given a decision task $\T = (A,u)$, the decision maker must pick an action $a\in A$ in order to maximize the payoff $u(a,y)\in \R$. Here, the payoff depends not only on the chosen action $a$, but also on the true outcome $y\in \{0,1\}$ unknown to the decision maker.
For example, if the outcome $y\in \{0,1\}$ represents whether or not it will be rainy today, a natural decision task may have two actions to choose from: $A = \{\text{take umbrella}, \text{not take umbrella}\}$. Each combination $(a,y)$ of action and outcome corresponds to a payoff value $u(a,y)$ that may depend on the susceptibility to rain and the inconvenience of carrying an umbrella.

Prediction enables decision making under uncertainty.
While the decision maker is unable to observe the true outcome $y$ before choosing the action, we assume that they are assisted by a prediction $v\in [0,1]$. 
In the ideal case, the prediction correctly represents the probability distribution of the true outcome. That is, the outcome $y$ follows the Bernoulli distribution with parameter $v$ (denoted $\y\sim v$). To maximize the expected payoff, the decision maker should choose the action
\begin{equation}
\label{eq:best-response}
\sigma_{\T}(v) \in \argmax_{a\in A}\E_{\y\sim v}u(a,\y)
\end{equation}
in response to the (correct) prediction $v$. We call the function $\sigma_\T:[0,1]\to A$ the \emph{best-response function}.
Throughout the section, we assume that each decision task $\T = (A,u)$ is associated with a well-defined best-response function. That is, we focus on tasks $\T$ where the $\argmax$ in \eqref{eq:best-response} is always non-empty.

In reality, predictions are seldom perfectly correct. It is thus unclear whether applying the best-response function would still lead to optimal payoff. The following theorem tells us that as long as the predictions are calibrated, the best response function remains the optimal mapping from predictions to actions, allowing the decision maker to \emph{trust the predictions as if they were correct}.

\begin{theorem}[Calibrated Predictions are Trustworthy]
\label{thm:cal-opt}
Let $\D$ be a joint distribution on $\X\times \{0,1\}$. For any perfectly calibrated predictor $p:\X\to [0,1]$ and any decision task $\T = (A,u)$, it holds that
\begin{equation}
\label{eq:cal-opt}
\E_{(\x,\y)\sim \D}[u(\sigma_\T(p(\x)),\y)] = \max_{\sigma:[0,1]\to A}\  \E_{(\x,\y)\sim \D}[u(\sigma(p(\x)),\y)].
\end{equation}
In other words, the maximum value of the expected payoff is attained when we choose $\sigma = \sigma_\T$. Conversely, if \eqref{eq:cal-opt} holds for every decision task $\T$, then the predictor $p$ is perfectly calibrated.
\end{theorem}
We defer the proof of \Cref{thm:cal-opt} to \Cref{sec:cha} and discuss how it suggests a new calibration  measure.
According to the theorem, if a predictor $p$ is miscalibrated, then the right-hand side of \eqref{eq:cal-opt} is larger than the left-hand side for some decision task $\T$. The difference between the two sides is exactly the payoff loss incurred by the decision maker who follows the best-response strategy $\sigma_\T$ assuming (incorrectly) that the predictions were calibrated. Thus, a natural measure of the level of miscalibration is exactly this payoff loss.
For a fixed decision task $\T$, this payoff loss is termed the \emph{Calibration Fixed Decision Loss (CFDL)} \cite{cdl}.
Taking the worst-case payoff loss over all decision tasks $\T = (A,u)$ with bounded payoff functions $u:A\to [0,1]$, we get the Calibration Decision Loss (CDL).

\begin{definition}[Calibration Decision Loss (CDL) \cite{cdl}]
\label{def:cdl}
Let $\D$ be a joint distribution over $\X\times \{0,1\}$. Given a predictor $p:\X\to [0,1]$, we define its \emph{Calibration Fixed Decision Loss (CFDL)} with respect to a (fixed) decision task $\T = (A,u)$ as
\[
\cfdl_{\T}(p,\D):= \max_{\sigma:[0,1]\to A}\E_{(\x,\y)\sim \D}[u(\sigma(p(\x)),\y)] - \E_{(\x,\y)\sim \D}[u(\sigma_\T(p(\x)),\y)].
\]
We define the \emph{Calibration Decision Loss (CDL)} of the predictor $p$ as the supremum of the CFDL over all decision tasks $(A,u)$ where the payoff function $u:A\to [0,1]$ has its range bounded in $[0,1]$:
\[
\cdl(p,\D):= \sup_{\T = (A,u), u:A\to [0,1]}\cfdl_\T(p,\D).
\]
\end{definition}
As we will see when we prove \Cref{thm:cal-opt} in \Cref{sec:cha}, the CDL is zero if and only if the predictor $p$ is perfectly calibrated. If a predictor is not perfectly calibrated but has a small CDL, any decision maker can still trust the predictor as if it were calibrated without losing too much expected payoff. This holds because the CDL is the supremum of the CFDL over \emph{all} payoff-bounded decision tasks. 

We note that in the definition of CDL, decision tasks are restricted to have a bounded payoff function $u:A\to [0,1]$. This restriction is only for the purpose of normalization: multiplying the payoff function by any positive constant changes the corresponding CFDL by the same constant factor, whereas adding a constant to the payoff function does not change the CFDL.
There is no further restriction on the decision tasks beyond bounded payoff functions. In particular, the action set $A$ can have arbitrary (even infinite) size. A small CDL implies that trusting the predictions will incur small payoff loss for \emph{all} such decision tasks.

	A natural question is how the CDL is related to other measures of the calibration error. We will prove that the CDL is quadratically related to the ECE:
	\begin{theorem}
	[\cite{KLST23,cdl}]
	\label{thm:relation}
	    Let $\D$ be a joint distribution over $\X\times \{0,1\}$. For any predictor $p:\X\to [0,1]$, 
	    \begin{equation}
	    \label{eq:relation}
	    \ECE(p,\D)^2 \le \ECE_2(p,\D) ^2\le \cdl(p,\D) \le 2\, \ECE(p,\D) \le 2\, \ECE_2(p,\D).
	    \end{equation}
	\end{theorem}
	Moreover, the quadratic relationship between $\cdl$ and $\ECE$ shown in \Cref{thm:relation} is tight (up to lower order terms): for any $\varepsilon\in (0,1/10)$, there exist two pairs $(p_1,\D_1),(p_2,\D_2)$ such that
	\begin{align*}
	& \ECE_2(p_1,\D_1) = \varepsilon,  & \cdl(p_1,\D_1) = 2\varepsilon;\\
	& \ECE(p_2,\D_2) = \varepsilon,  & \cdl(p_2,\D_2) \le \varepsilon^2 + O(\varepsilon^3).
	\end{align*}

We defer the proof of \Cref{thm:relation} to \Cref{sec:relation}. Here we briefly describe the two tight examples. The first example $(p_1,\D_1)$ is very simple. For $(\x,\y)\sim\D_1$, we draw $\y\in \{0,1\}$ from the Bernoulli distribution with parameter $1/2 + \varepsilon$, independent of $\x$. The predictor $p_1$ is the constant predictor $p_1(x) = 1/2$.
In the second example, we draw $\x$ uniformly at random from the interval $[\varepsilon, 1]$ and then draw $\y\in \{0,1\}$ from the Bernoulli distribution with parameter $\x - \varepsilon$. The predictor $p_2$ is the identity function $p_2(x) = x$ for $x\in [\varepsilon,1]$. We will prove the correctness of the examples in \Cref{sec:proof-examples}.

The second example, $(p_2,\D_2)$, demonstrating that the CDL can be significantly smaller than the ECE, is quite instructive. It opens up the possibility that the CDL can be minimized at a faster rate than what is possible for ECE in the online setting. Indeed, the main technical result of \cite{cdl} gives an efficient online CDL minimization algorithm achieving rate $O(\sqrt T \log T)$, overcoming the information-theoretic lower bound $\Omega(T^{0.54389})$ for ECE \cite{sidestep,breaking} (see \Cref{sec:online} for more discussions).

To conclude this subsection, CDL measures the calibration error using the payoff loss of downstream decision makers caused by mis-calibration. In addition to introducing CDL as a meaningful decision-theoretic measure of calibration, the work of \cite{cdl} also shows that CDL allows a significantly better rate than what is possible for ECE in online calibration, which we discuss in \Cref{sec:online}.

In \Cref{sec:cha} we give a simpler yet equivalent definition of the CFDL in \eqref{eq:cfdl-2}, which leads to an interpretation of CDL through the lens of indistinguishability.

\subsection{Characterization of the Maximum Expected Payoff}
\label{sec:cha}
In this section we prove \Cref{thm:cal-opt}. We start by giving a characterization of the maximum expected payoff on the right-hand side of \eqref{eq:cal-opt} for a general predictor $p$ that may or may not be calibrated, which simplifies the definition of CFDL and will be useful in the proof.

Recall the definition of the recalibration $\recal$ of $p$ (\Cref{def:recal}): $\recal$ is obtained by replacing each prediction value $v = p(x)$ with the actual conditional expectation $\E[\y|p(\x) = v]$. Clearly, $\recal$ is perfectly calibrated. If $p$ is perfectly calibrated, then $\recal = p$. 
We have the following characterization of the maximum expected payoff achievable by post-processing $p$:
\begin{lemma}
\label{lm:recal}
Let $\D$ be a joint distribution on $\X\times \{0,1\}$. For any predictor $p:\X\to [0,1]$ and any decision task $\T = (A,u)$, it holds that
\begin{equation}
\label{eq:recal}
\max_{\sigma:[0,1]\to A}\ \E_{(\x,\y)\sim \D}[u(\sigma(p(\x)),\y)] = \E_{(\x,\y)\sim \D}[u(\sigma_\T(\recal(\x)),\y)],
\end{equation}
where $\recal$ is the recalibration of $p$.
\end{lemma}
\begin{proof} 
The lemma can be proved by considering the level sets $\X_v:= \{x\in \X:p(x) = v\}$ for $v\in [0,1]$. Within each level set, $p$ is a constant function, and the functions $\sigma(p(x))$ formed by all choices of $\sigma:[0,1]\to A$ are all the constant functions on this level set taking value in $A$.
Moreover, for any level set $\X_v$, the conditional distribution of $\y$ given $\x\in \X_v$ is the Bernoulli distribution with parameter $\recal(\x)$, where $\recal(x)$ is also a constant function for $x\in \X_v$.
Decomposing \eqref{eq:recal} by the level sets, the lemma follows from the definition of the best-response function $\sigma_\T$ in \eqref{eq:best-response}.
\end{proof}

We can now rewrite the definition of CFDL (\Cref{def:cdl}) based on \Cref{lm:recal}:
\begin{equation}
\label{eq:cfdl-2}
\cfdl_{\T}(p,\D) = \E_{(\x,\y)\sim \D}[u(\sigma_\T(\recal(\x)),\y)] - \E_{(\x,\y)\sim \D}[u(\sigma_\T(p(\x)),\y)].
\end{equation}
This expression allows us to easily calculate the CFDL for specific decision tasks. For example, consider the task $\T_2 = (A,u)$ where the action space $A$ is the unit interval $A = [0,1]$, and the payoff function is quadratic: 
\[
u(a,y) = 1 - (a - y)^2 \in [0,1], \quad \text{for $a\in [0,1]$ and $y\in \{0,1\}$.}
\]
The corresponding best-response function is the identity: $\sigma_\T(v) = v$. Plugging it in \eqref{eq:cfdl-2}, we obtain an equality between the CFDL and the square of $\ECE_2$:
\begin{align}
\cfdl_{\T_2}(p,\D) & = \E_{(\x,\y)\sim \D}[(p(\x) - \y)^2 - (\recal(\x) - \y)^2] \notag\\
&= \E[p(\x)^2 - \recal(\x)^2 + 2 \y (\recal(\x) - p(\x))]\notag\\
& = \E[p(\x)^2 - \recal(\x)^2 + 2 \recal(\x) (\recal(\x) - p(\x))] \tag{$\E[\y | \recal(\x), p(\x)] = \recal(\x)$}\notag\\
&= \E[(p(\x) - \recal(\x))^2] = \ECE _2(p,\D)^2.\label{eq:quadratic}
\end{align}
We are now ready to prove \Cref{thm:cal-opt}.
\begin{proof}[of \Cref{thm:cal-opt}]
If $p$ is perfectly calibrated, then $p = \recal$, and \eqref{eq:cal-opt} follows immediately from \Cref{lm:recal}.
For the reverse direction, if \eqref{eq:cal-opt} holds for any decision task, then in particular, it holds for the task $\T_2$ above, implying $\cfdl_{\T_2}(p,\D) = 0$. By \eqref{eq:quadratic}, we have $\ECE_2(p,\D) = 0$, so $p$ is perfectly calibrated, as desired. Since the quadratic payoff function of $\T_2$ has a bounded range $[0,1]$, this proof also implies that the CDL of a predictor is zero if and only if the predictor is perfectly calibrated. 
\end{proof}

\subsection{The Bregman Divergence View of CDL}
\label{sec:div}
We show that the CFDL of a predictor $p$ w.r.t.\ any decision task $\T$ can be expressed as a Bregman divergence $D_\varphi(J^*\|J^p)$ between the two joint distributions $J^*$ and $J^p$ (\Cref{thm:bregman}). Our proof uses a classic characterization of \emph{proper scoring rules} \cite{mccarthy,savage,gneiting}. 

We start with the definition of Bregman divergence.

\begin{definition}[Bregman Divergence]
    Let $\varphi:[0,1]\to \R$ be a convex function and let $\nabla \varphi:[0,1]\to \R$ be its subgradient. For any pair of values $\mu^*,\mu\in [0,1]$, their \emph{Bregman divergence} w.r.t.\ $\varphi$ is
    \[
    D_{\varphi}(\mu^*\|\mu) := \varphi(\mu^*) - \varphi(\mu) - \nabla \varphi (\mu) \cdot (\mu^* - \mu).
    \]
    Since $\nabla \varphi(\mu)$ is a subgradient of $\varphi$ at $\mu$, the Bregman divergence is always nonnegative. When $\mu = \mu^*$, the Bregman divergence becomes zero.
\end{definition}

We will interpret the values $\mu^*,\mu\in [0,1]$ in the definition above as the parameters of two Bernoulli distributions.
For example, if we choose $\varphi(\mu)$ to be the negative Shannon entropy of the Bernoulli distribution with parameter $\mu$: 
\[
\varphi(\mu) = \mu\ln \mu - (1 - \mu)\ln (1 - \mu),
\]
then the Bregman divergence becomes the KL divergence between the two Bernoulli distributions parameterized by $\mu^*$ and $\mu$:
\[
D_{\varphi}(\mu^*\|\mu) = \mu^*\ln \frac {\mu^*}\mu + (1 - \mu^*)\ln \frac{1 - \mu^*}{ 1 - \mu}.
\]
The following key theorem makes the connection between Bregman divergences and decision tasks.

\begin{theorem}
\label{thm:proper}
    For any decision task $\T = (A,u)$, there exists a  convex function $\varphi:[0,1]\to \R$ with subgradient $\nabla \varphi:[0,1]\to \R$ such that
    \[
    u(\sigma_\T(v),y) = \varphi(v) + \nabla\varphi(v) \cdot (y - v) \quad \text{for every $v\in [0,1]$ and $y\in \{0,1\}$.}
    \]
\end{theorem}
To prove the theorem, one should first observe that the function $u(\sigma_\T(v),y)$ is a \emph{proper scoring rule}. 
That is, for any $v,v'\in [0,1]$, we have
\[
\E_{\y\sim v}u(\sigma_\T(v),\y) \ge \E_{\y\sim v}u(\sigma_\T(v'),\y),
\]
which follows from the definition \eqref{eq:best-response} of the best-response function $\sigma_\T$. The theorem then follows from a standard characterization of proper scoring rules \cite{mccarthy,savage,gneiting}.

We can now write the expected payoff achieved by a predictor $p$ using the Bregman divergence between $p$ and its recalibration $\recal$:
\begin{lemma}
\label{lm:payoff-bregman}
    Fix a joint distribution $\D$ of $(\x,\y)\in \X\times \{0,1\}$. Let $p:\X\to [0,1]$ be a predictor and $\recal$ be its recalibration (\Cref{def:recal}). Then for any decision task $\T = (A,u)$ and the corresponding convex function $\varphi$ from \Cref{thm:proper},
    \begin{align}
    \E_{\D}[u(\sigma_\T(p(\x)),\y)] & = \E_{\D}[\varphi(\recal(\x))]- \E_{\D}[D_\varphi(\recal(\x)\|p(\x))], \label{eq:payoff-bregman-1}\\
    \cfdl_{\T}(p,\D) & = \E_{\D}[D_\varphi(\recal(\x)\|p(\x))].\label{eq:payoff-bregman-2}
    \end{align}
\end{lemma}
\begin{proof}
    By \Cref{thm:proper},
    \begin{align*}
    \E_\D[u(\sigma_\T(p(\x),\y))] & = \E_\D[\varphi(p(\x)) + \nabla \varphi(p(\x))\cdot (\y - p(\x))]\\
    & = \E_\D[\varphi(p(\x)) + \nabla \varphi(p(\x))\cdot (\hat p(\x) - p(\x))] \tag{because $\E[\y|p(\x)] = \hat p(\x)$}\\
    & = \E_\D[\varphi(\hat p(\x))] - \E_\D[\varphi(\hat p(\x)) - \varphi(p(\x)) - \nabla \varphi(p(\x))\cdot (\hat p(\x) - p(\x))]\\
    & = \E_\D[\varphi(\hat p(\x))] - \E_\D[D_\varphi(\hat p(\x)\| p(\x))].
    \end{align*}
This proves Equation \eqref{eq:payoff-bregman-1}. Similarly,
\[
\E_\D[u(\sigma_\T(\hat p(\x),\y))] = \E_\D[\varphi(\hat p(\x))] - \E_\D[D_\varphi(\hat p(\x)\| \hat p(\x))] = \E_\D[\varphi(\hat p(\x))].
\]
Taking the difference between the two equations above, we have
\[
\cfdl_{\T}(p,\D) = \E_\D[u(\sigma_\T(\hat p(\x),\y))] - \E_\D[u(\sigma_\T(p(\x),\y))] = \E_\D[D_\varphi(\hat p(\x)\| p(\x))].
\]
This proves Equation \eqref{eq:payoff-bregman-2}.
\end{proof}

We now generalize the definition of Bregman divergence to joint distributions, such as $J^*$ and $J^p$, over the domain $[0,1]\times \{0,1\}$.

\begin{definition}[Induced Bregman Divergence between Joint Distributions]
\label{def:induced-bregman}
    Let $\varphi:[0,1]\to \R$ be a convex function and let $\nabla \varphi:[0,1]\to \R$ be its subgradient.
    For any joint distribution $J$ of $(\v,\y)\in [0,1]\times \{0,1\}$, we use $\mu_J(\v) = \E_J[\y|\v]\in [0,1]$ to denote the parameter of the Bernoulli distribution of $\y$ conditioned on $\v$.
    Let $J_1,J_2$ be a pair of joint distributions of $(\v,\y)\in [0,1]\times \{0,1\}$ that share the same marginal distribution of $\v$  and denote this marginal distribution by $M$. 
    We define the \emph{Bregman divergence} between $J_1$ and $J_2$ \emph{induced by} $\varphi$ as\footnote{    One can also view $D_\varphi(J_1\|J_2)$ as the Bregman divergence corresponding to the negative entropy $\Phi(J)$ of any joint distribution $J$ of $(\v,\y)\in [0,1]\times \{0,1\}$ defined by
    $
    \Phi (J):=\E_{(\v,\y)\sim J}[\varphi(\mu_J(\v))].
    $} \[
    D_\varphi(J_1\|J_2) := \E_{\v\sim M}[D_\varphi(\mu_{J_1}(\v)\|\mu_{J_2}(\v))].
    \]
\end{definition}

Combining \Cref{lm:payoff-bregman} and \Cref{def:induced-bregman}, we have a Bregman divergence characterization of the CFDL for any decision task $\T$.
\begin{theorem}[Bregman Divergence View of CFDL]
\label{thm:bregman}
Let $\D$ be a joint distribution over $\X\times \{0,1\}$, and let $p:\X\to [0,1]$ be a predictor. As before, given $(\x,\y^*)\sim \D$, we draw $\y^p$ from the Bernoulli distribution with parameter $p(\x)$, and use $J^*,J^p$ to denote the distributions of $(p(\x),\y^*)$ and $(p(\x),\y^p)$, respectively. Then for any decision task $\T = (A,u)$ and the corresponding convex function $\varphi$ from \Cref{thm:proper},
    $\cfdl_\T(p,\D) = D_{\varphi}(J^*\|J^p)$.
\end{theorem}
\begin{proof}
Let $\hat p$ be the recalibration of $p$ (\Cref{def:recal}). By the definitions of $J^*$ and $J^p$, for any $x\in \X$, we have 
\begin{align}
\mu_{J^*}(p(x)) & = \hat p(x), \label{eq:muJstar}\\
\mu_{J^p}(p(x)) & = p(x).\label{eq:muJp}
\end{align}
Let $M$ denote the marginal distribution of $p(\x)$\ where $(\x,\y^*)\sim \D$. By \Cref{lm:payoff-bregman},
\begin{align*}
\cfdl_{\T}(p,\D) & = \E_\D[D_\varphi(\hat p(\x)\| p(\x))]\\
& = \E_{\v\sim M}[D_\varphi(\mu_{J^*}(\v)\| \mu_{J^p}(\v))]\\
& = D_{\varphi}(J^*\|J^p).
\end{align*}
\end{proof}

\subsection{V-shaped Divergences}
In this subsection, we discuss a fundamental result about Bregman divergences (\Cref{thm:v-shape}) that will be used to prove \Cref{thm:relation}.

CDL focuses on decision tasks $\T = (A,u)$ with $[0,1]$-bounded payoff functions $u:A\to [0,1]$. For such tasks, the corresponding convex function $\varphi$ from \Cref{thm:proper} must have bounded subgradients:
\begin{equation}
\label{eq:bounded-subgradient-0}
\nabla \varphi(v) = u(\sigma_\T(v),1) - u(\sigma_\T(v),0)\in [-1,1] \quad \text{for every $v\in [0,1]$.}
\end{equation}
While the convex functions $\varphi$ with bounded subgradients $\nabla \varphi(v)\in [-1,1]$ form a large family, a fundamental result by \cite{optimization-scoring-rule}, which we include as \Cref{thm:v-shape} below, shows that the divergences $D_\varphi$ defined by this family can be captured by extremely simple functions $\varphi$ that are termed \emph{V-shaped} functions.
Specifically, for each $\vstar\in [0,1]$, a V-shaped function $\varphi_\vstar$ is defined as follows:
\[
\varphi_{\vstar}(v) = |v - \vstar| \quad \text{for every $v\in [0,1]$.}
\]
The Bregman divergence $D_{\varphi_\vstar}$ is correspondingly termed a \emph{V-shaped} divergence, and it can be easily computed as follows: for $v_1,v_2\in [0,1]$, we have 
\begin{equation}
\label{eq:v-breg}
D_{\varphi_\vstar}(v_1\|v_2) = \begin{cases}
2|v_1 -v^*| \le 2 |v_1 - v_2|, & \text{if }\vstar \in (v_1,v_2] \text{ or if }\vstar \in (v_2,v_1];\\
0, & \text{otherwise.}
    \end{cases}
\end{equation}
The following theorem gives an upper bound on the expected divergence $D_\varphi$ for a general $\varphi$ with bounded subgradient in terms of V-shaped divergences $D_{\varphi_{\vstar}}$.
\begin{theorem}[\cite{optimization-scoring-rule}]
\label{thm:v-shape}
    Let $\varphi:[0,1]\to \R$ be a convex function whose subgradient is bounded: $\nabla \varphi(v) \in [-1,1]$ for every $v\in [0,1]$. Then for any distribution $\Pi$ of $(v_1,v_2)\in [0,1]$,
    \[
    \E_{(v_1,v_2)\sim \Pi}D_\varphi(v_1,v_2) \le \sup_{\vstar\in [0,1]}\E_{(v_1,v_2)\sim \Pi}D_{\varphi_\vstar }(v_1,v_2).
    \]
\end{theorem}

\subsection{Relationship to ECE}
\label{sec:relation}
We prove \Cref{thm:relation}, which demonstrates the quadratic relationship between the CDL and the ECE. The first and last inequalities in \eqref{eq:relation} follow immediately from Jensen's inequality. The second inequality can be proved using the decision task $\T_2$ from \Cref{sec:cha}. Specifically, the payoff function of $\T_2$ has its range bounded in $[0,1]$, so by the definition of CDL and Equation \eqref{eq:quadratic},
\[
\cdl(p,\D) \ge \cfdl_{\T_2}(p,\D) = \ECE_2(p,\D)^2.
\]
Now we prove the third inequality in \eqref{eq:relation}. 
By \Cref{lm:payoff-bregman} and \Cref{thm:v-shape}, for any decision task $\T$ with $[0,1]$ bounded payoffs,
\begin{align*}
\cfdl_\T(p,\D) = \E_\D[D_\varphi(\recal(\x)\| p(\x))] & \le \sup_{\vstar\in [0,1]} \E_\D[D_{\varphi_\vstar}(\recal(\x)\| p(\x))]\\
&\le 2 \E_\D|\recal(\x) - p(\x)| = \ECE(p,\D). \tag{by \eqref{eq:v-breg}}
\end{align*}
This proves $\cdl(p,\D) \le 2 \ECE(p,\D)$, as desired.

\subsection{Tight examples between CDL and ECE}
\label{sec:proof-examples}
We prove the correctness of the two examples $(p_1,\D_1),(p_2,\D_2)$ we mentioned after \Cref{thm:relation} that shows the tightness of \Cref{thm:relation}.

In the first example, we have $p_1(x) = 1/2$ and $\hat p_1(x) = 1/2 + \varepsilon$ for any $x$, so it is clear that $\ECE_2(p_1,\D_1) = \varepsilon$. To prove $\cdl(p_1,\D_1) \ge 2\varepsilon$, consider the task $\T_1 = (A,u)$ with two actions: $A = \{0,1\}$. The payoff function $u$ is defined such that $u(a,y)= 1$ if $a = y$, and $u(a,y) = 0$ otherwise. The best-response function is $\sigma_\T(v) = 0$ if $v\le 1/2$, and $\sigma_T(v) = 1$ otherwise. We have
\begin{align*}
\E[u(\sigma_{\T_1}(p_1(\x)),\y)] = \E[u(0,\y)] = \Pr[\y = 0] = \frac 12 - \varepsilon,\\
\E[u(\sigma_{\T_1}(\hat p_1(\x)),\y)] = \E[u(1,\y)] = \Pr[\y = 1] = \frac 12 + \varepsilon.
\end{align*}
Taking the difference between the two expected payoffs, we get $\cdl(p_1,\D_1) \ge \cfdl_{\T_1}(p_1,\D_1) = 2\varepsilon$.

In the second example, we have $p_2(x) = x$ and $\hat p_2(x) = x - \varepsilon$, so it is clear that $\ECE(p_2,\D_2) = \varepsilon$. Now we prove 
\begin{equation}
\label{eq:example-2}
\cdl(p_2,\D_2) \le \frac{\varepsilon^2}{1 - \varepsilon} = \varepsilon^2 + O(\varepsilon^3).
\end{equation}
Consider any decision task $\T = (A,u)$ with a $[0,1]$-bounded payoff function $u:A\to [0,1]$. 
By \Cref{lm:payoff-bregman} and \Cref{thm:v-shape},
\begin{align}
\cfdl_{\T}(p_2,\D_2) & = \E_{\D_2}[D_\varphi(\hat p_2(\x)\| p_2(\x))]\notag\\
& = \E_{\D_2}[D_\varphi(\x - \varepsilon\| \x)]\notag\\
& \le \sup_{\vstar\in [0,1]}\E_{\D_2}[D_{\varphi_\vstar}(\x - \varepsilon\| \x)]\notag\\
& = 2\sup_{\vstar\in [0,1]}\E_{\D_2}\Big[(\vstar - (\x - \varepsilon))\mathbb I[\vstar\in (\x - \varepsilon,\x]]\Big]\notag \tag{by \eqref{eq:v-breg}}\\ 
& = \frac 2{1 - \varepsilon}\sup_{\vstar\in [0,1]}\int_\varepsilon ^1(\vstar - (x - \varepsilon))\ind[\vstar\in (x - \varepsilon,x]]\dd x\notag\\ 
& \le \frac{2}{1 - \varepsilon}\sup_{\vstar\in [0,1]}\int_{v^*}^{v^* + \varepsilon}(\vstar - (x - \varepsilon)) \dd x\notag\\ 
& = \frac{\varepsilon^2}{1 - \varepsilon}.
\end{align}
Since this upper bound on the $\cfdl$ holds for any decision task $\T$ with a $[0,1]$-bounded payoff function, it implies \eqref{eq:example-2}, as desired.

\subsection{Further Work}
As we discuss in this section, the defining property of CDL is that it provides a meaningful guarantee on the swap regret incurred by downstream decision makers who trust the predictions. However, CDL is undesirable in other aspects: like ECE, it is discontinuous and requires high sample complexity to estimate. Recent work of \cite{test-action} introduces the notion of \emph{cutoff calibration error} to address the sample complexity issue while maintaining a restricted form of the decision-theoretic guarantee of CDL (e.g.\ they consider the regret relative to \emph{monotone} post-processings of the predictions). This notion of cutoff calibration is essentially identical to the notion of \emph{proper calibration} from \cite{optimal-omni}, who give an algorithm achieving $\widetilde O(\sqrt T)$ error rate for proper calibration in the online setting (see \Cref{sec:online} for the setting). 
The works of \cite{when-does,BlasiokN24,smooth-decision} show that  low smooth calibration error also gives certain decision-theoretic guarantees. In particular, these works  show that it  implies low regret for certain forms of Lipschitz post-processings or for decision makers who make randomized responses (e.g.\ by adding noise to the predictions), though this implication often comes with a quantitative loss (e.g.\ smooth calibration error being at most $\varepsilon$ only implies an  $O(\sqrt \varepsilon)$ regret).

\section{Online Calibration}

\label{sec:online}
We have discussed a variety of ways to quantify the calibration error of a given predictor. In this section, we turn to the algorithmic question of \emph{constructing} a predictor with low calibration error.
This question, when naively formulated, admits a trivial and unenlightening solution: one can simply construct a constant predictor that assigns (an approximation of) the overall average $\E[y]$ to every individual $x$. This is a well-calibrated predictor according to every calibration measure we have discussed. Thus, for the algorithmic question to be insightful, it is essential to formulate it in such a way that reaches beyond the trivial solution.
The seminal work of Foster and Vohra \cite{FosterV98} introduced one such interesting question that turned into an active area of research with exciting recent progress: calibration in \emph{online} prediction.
We will first describe the problem setting and then briefly survey some key results in the literature.

The online prediction problem has $T$ rounds indexed by $t \in [T]$. In round $t$, our algorithm makes a prediction $p_t\in [0,1]$, and nature reveals an outcome $y_t\in \{0,1\}$. For example, we can interpret the problem as predicting the chance of rain each day for $T$ days, where $p_t$ is the prediction we make on day $t$, and $y_t = 1$ if day $t$ is rainy. Since the rounds are ordered chronologically, we allow our algorithm to choose $p_t$ as a function of the history $H_{t-1} = (p_1,\ldots,p_{t-1},y_1,\ldots,y_{t-1})$, and similarly, $y_t$ can depend on the history $H_{t-1}$ as well.

To evaluate the calibration error of the prediction sequence $p_{1,\ldots,T}:= (p_1,\ldots,p_T)$ w.r.t.\ the outcome sequence $y_{1,\ldots,T}:=(y_1,\ldots,y_T)$,  we consider the predictor $p:\{1,\ldots,T\}\to [0,1]$ that assigns prediction $p(t) := p_t$ to each time step $t = 1,\ldots,T$.
Viewing each time step as an individual, we let $\D$ be the uniform distribution over the individual-outcome pairs $(t,y_t)$ for $t = 1,\ldots,T$.
By slight abuse of notation, we  can transform any calibration measure $\CAL$ for $(p,\D)$ into a calibration measure $\CAL$ for $(p_{1,\ldots,T}, y_{1,\ldots,T})$ as follows:
\[
\CAL(p_{1,\ldots,T}, y_{1,\ldots,T}) := T\, \CAL(p,\D).
\]

Once a calibration measure $\CAL$ is chosen, our goal is to design a prediction algorithm that guarantees a small (e.g.\ sub-linear, i.e., $o(T)$) calibration error according to $\CAL$, regardless of how the outcomes $y_t$ are generated. We wish to design a prediction algorithm $P$ that specifies how $p_t$ should be chosen as a function of the history $H_{t-1}$ for every round $t$. We want the calibration error to be small regardless of nature's strategy $Y$, which specifies how $y_t$ should be chosen as a function of $H_{t-1}$ for every round $t$. That is, we want to solve the following optimization problem:
\[
\minimize_{P}\max_Y \CAL(p_{1,\ldots,T}, y_{1,\ldots,T}), \text{ where $p_{1,\ldots,T}, y_{1,\ldots,T}$ is generated by $P$ and $Y$.}
\]

For some calibration measures (e.g.\ ECE and CDL), it is necessary to use randomized prediction algorithms to achieve sub-linear rates. Such an algorithm constructs a distribution $\P$ over prediction strategies $P$ to solve the following problem:
\[
\minimize_{\P}\max_Y \E_{P\sim \P}[\CAL(p_{1,\ldots,T}, y_{1,\ldots,T})].
\]

Here is why randomized predictions are necessary for achieving sub-linear rates for ECE or CDL. For every deterministic prediction algorithm $P$, nature can infer the prediction $p_t$ based on the history $H_{t - 1}$, and can then choose $y_t = 1$ if and only if $p_t < 1/2$, incurring an $\Omega(T)$ rate for ECE and CDL.

In Table~\ref{table:rates}, we summarize the current best upper and lower bounds on the optimal online calibration rates for a few calibration error measures we discussed earlier, which is an active topic for recent research. Notably, the only calibration measure in this table that does not allow an $\widetilde O (\sqrt T)$ rate is ECE.

There are substantial gaps between the best upper and lower bounds for many calibration measures in this table, making it a natural question to close or reduce these gaps. Very recently, the works of \cite{peng} and \cite{Fishelson} have achieved significant progress on online calibration algorithms in the \emph{multi-class} setting, opening up another exciting area for future research.

\begin{table}
\begin{center}

\begin{tabular}{lll} 
 \toprule
\textbf{Calibration Error} & \textbf{Rate Upper Bound} & \textbf{Rate Lower Bound} \\ [0.5ex] 
 \toprule
 \makecell[l]{Expected Calibration Error \\ (ECE)} & \makecell[l]{$O(T^{2/3})$  \cite{FosterV98} \\ \\ $O(T^{2/3 - \varepsilon})$ \cite{breaking}}  & \makecell[l]{$\Omega(T^{1/2})$ {[Folklore]} \\ $\Omega(T^{0.528})$ \cite{sidestep} \\ $\Omega(T^{0.54389})$  \cite{breaking}} \\ 
\midrule
\makecell[l]{Distance to Calibration \\ \cite{BlasiokGHN23}} & \makecell[l]{$O(T^{1/2})$  \cite{dist-online} \cite{elementary}} & \makecell[l]{$\Omega(T^{1/3})$ \cite{dist-online}} \\ 
 \midrule
 \makecell[l]{Smooth Calibration Error \\ \cite{kakadeF08}} & \makecell[l]{$O(T^{1/2})$ \cite{dist-online}  \cite{elementary}} &  \makecell[l]{$\Omega(T^{1/3})$ \cite{dist-online}}\\
 \midrule
 \makecell[l]{Calibration Decision Loss  \\ (CDL) \cite{cdl}} & $O(T^{1/2}\log T)$ \cite{cdl} & $\Omega(T^{1/2})$ \cite{cdl}\\
 \bottomrule
\end{tabular}
\end{center}
\caption{Upper and lower bounds on the optimal rates for online calibration}
\label{table:rates}
\end{table}

\section{The Distance to Calibration}
\label{sec:dtc}

At this point, we seem to have a Cambrian explosion of approximate calibration measures, each of which has their own desirable properties, and will give different calibration errors for a predictor. How should we compare these different measures, and decide which to use? Is there any notion of ground truth, that would guide us in this choice? In this section, we present one possible answer to this question via the notion of the distance to calibration \cite{BlasiokGHN23}.  We show that the smooth calibration error gives us the best approximation to this ground-truth measure in an information-theoretic sense. 

\eat{
\ilcomment{A bit of an awkward transition}
In Section \ref{sec:cal-error}, we discussed an indistinguishability-based measure of calibration: weighted calibration error. The flexibility in the choice of the weight family $W$ allows it to express a broad class of calibration measures. This flexibility and expressivity demonstrates the advantage of the indistinguishability perspective, but it also leads to a Cambrian explosion of calibration measures. Our challenge now becomes how to choose a specific calibration measure from this broad class. In this section, we posit a candidate \emph{ground-truth} measure of calibration: the \emph{distance to calibration}. 
}

Recall that we defined $\D^*$ to be the joint distribution of $\x, \y^*$, whereas $J^*$ denotes the joint distribution $(p(\x), \y^*)$. 

\begin{definition}
[Distance to calibration \cite{BlasiokGHN23}]
Given a distribution $\D^*$,  define
$\Cal(\D^*)$ to be the set of predictors $q:\X \to [0,1]$ such that $q$ is perfectly calibrated under $\D^*$. Define the \emph{true distance to calibration} of the predictor $p$ as
\[ \dCE(p, \D^*) = \min_{q \in \Cal(D^*)} d(p, q).\]
\end{definition}

This definition formalizes the intuition that a predictor which can be made perfectly calibrated by a small change to its predictions is close to being calibrated. 
A desirable property that follows immediately from this definition is that the distance to calibration is continuous (unlike ECE). In fact, $\dCE$ is Lipschitz continuous: if we chance our predictor $p$ to a different predictor $p'$ that is $\varepsilon$-close to $p$ ($|d(p,p') \le \varepsilon|$), the distance to calibration can only change by at most $\varepsilon$ ($|\dCE(p,\D^*) - \dCE(p',\D^*)|\le \varepsilon$). This continuity property can be easily proved using the triangle inequality for the metric $d$.

Despite its intuitiveness and continuity, $\dCE$ differs from the other notions of calibration we have seen so far in a crucial way: it depends on the feature space $\X$ (at least, syntactically). This dependence comes about because both the set $\Cal(\D^*)$ of perfectly calibrated predictors and the distance metric $d$ depend on $\X$. The definition of $\dCE$ does not give any hints about how one might go about computing or approximating it.

It is natural to ask to what extent $\dCE$ actually depends on the space $\X$, and if it can be approximated by a calibration measure which is independent of $\X$. This leads us to two new definitions. 

\begin{definition}
[\cite{BlasiokGHN23}]
    The \emph{upper distance to calibration} $\uCE(J^*)$ is the maximum of $\dCE(p', \D')$ over all spaces $\X'$, distributions $\D'$ on $\X' \times \zo$ and predictors $p':\X' \to [0,1]$ such that the distribution  $J' = (p'(\x'), \y')$ is identical to the distribution $J^* = (\p(\x), \y^*)$. The \emph{lower distance to calibration} $\lCE$ is defined analogously, replacing the maximum by minimum.
\end{definition}

By their definition, both $\lCE$ and $\uCE$ achieve the goal of only depending on $J^*$ and not $\D^*$. It also follows that 
\[ \lCE(J^*) \leq \dCE(p, \D^*) \leq \uCE(J^*). \]
This leads to two questions:
\begin{enumerate}
\item The definitions of $\lCE$ and $\uCE$ seem rather cumbersome at first, since they involve optimizing over a possibly infinite collection of feature spaces and predictors. Are there more tractable characterizations of these notions, ideally ones that will let us estimate them efficiently?
\item How far apart are $\lCE$ and $\uCE$? An ideal situation would be that they are always equal, or at most a constant factor apart. If so, either of them could serve as a good approximation for $\dCE$, assuming we find efficient ways to compute them. 
\end{enumerate}

In the following subsection, we will show that the largest gap between the upper and lower distances is quadratic ($\uCE(J^*) \le 4\sqrt{\lCE(J^*)}$), and that the smooth calibration error gives a constant-factor approximation to the lower distance to calibration. Together, these results let us efficiently approximate the distance to calibration using smooth calibration error, as  in the work of \cite{hu2024testing}.

\subsection{Characterizing and Relating the Upper and Lower Distances to Calibration}

In this subsection, we answer the two questions above. Specifically, we give simple characterizations for the upper and lower distances in Theorems~\ref{thm:char-upper} and \ref{thm:char-lower}. We show that the two distances are at most quadratically apart in \Cref{thm:int-upper}.

We first give a simpler characterization of the upper distance. 
We begin with some definitions needed to state the characterization.
\begin{definition}[Calibrated post-processing]
 Define the set $K(J^*)$ to be the set of post-processing functions that, when applied to $p$, give a perfectly calibrated predictor. Formally, $ 
    K(J^*) = \{\kappa:[0,1] \to [0,1] \ s.t. \ (\kappa(p(\x)), y^*) \ \text{is perfectly calibrated.} \}$
\end{definition}

We observe that the set $K(J^*)$ is non-empty, since the constant predictor which predicts $\E[y^*]$ is calibrated, and this corresponds to the constant function $\kappa^{\textsf{av}}(v) = \E[\y^*]$ for all $v$. A more interesting post-processing is $\kappa^{\textsf{recal}}(v) = \E[\y^*|v]$, and we call the post-processed predictor $\recal(\x):= \kappa^{\textsf{recal}}( p(\x))$ the \emph{recalibration} of $p$: this predictor keeps the same level sets as $p$, and changes the predictions to be calibrated. 

\begin{definition}[Recalibration]
\label{def:recal}
Fix a distribution $\D$ of $(\x,\y)\in \X\times \{0,1\}$.
We define the \emph{recalibration} of a predictor $p:\X\to [0,1]$ to be another predictor, denoted by $\recal:\X\to [0,1]$, where $\recal(\x):= \E_\D[\y|p(\x)]$.  
\end{definition}

\begin{lemma}
\label{exercise:ece}
    It holds that $\ECE(p, \D^*)  = d(p ,\recal) = d(p,\kappa^{\textsf{recal}}\circ p)$, where $\circ$ denotes function composition.
\end{lemma}

In general, the set $K(J^*)$ could be much richer and possibly induce {\em closer} calibrated predictors. In particular, there often exist post-processings $\kappa\in K(J^*)$ such that $d(p,\kappa\circ p)$ is much smaller than $ d(p,\kappa^{\textsf{recal}}\circ p) = \ECE(p,\D^*)$. For the two point distribution $\D_2$ considered before, we have seen that $\ECE(p, \D_2) = 1/2 - \eps$ whereas it follows that $\kappa^{av} = 1/2$ and $d(p, 1/2) = \eps$. 

\cite{BlasiokGHN23} give the following characterization of the upper distance.
\begin{theorem}
\label{thm:char-upper}
\cite{BlasiokGHN23}
    We have
    \[ \uCE(J^*) = \min_{\kappa \in K(J^*)} d(p,\kappa\circ p) = \min_{\kappa \in K(J^*)}\E_{\x}|\kappa(p(\x)) - p(\x)|. \]
\end{theorem}
This theorem tells us that the upper distance of a given predictor $p$ is exactly its distance to the closest perfectly calibrated predictor that can be obtained by applying a post-processing $\kappa$ to $p$.

Let us sketch the proof idea. $K(J^*)$ is the set of relabelings of the level sets of $p$ which result in a calibrated predictor. For any space $X'$, distribution $D'$ and predictor $p'$ where $J' = J^*$, applying the post-processing function $\kappa \in J^*$ results in a perfectly calibrated predictor $\kappa(p')$ on $\X'$. Hence the distance from such predictors is always an upper bound on $\uCE$. For the space $X''$ where each level set is a single point, these are the only calibrated predictors, so the bound is tight.

We now turn to the lower distance. The good news is that the characterization is in terms of a calibration measure that we have encountered previously: the smooth calibration error $\smCE(p, \D^*)$. The proof however is more involved, we refer the reader to \cite{BlasiokGHN23,BlasiokN24}. 

\begin{theorem}
[\cite{BlasiokGHN23}]
\label{thm:char-lower}
We have
\[ \smCE(p, \D^*)/2 \leq \lCE(J^*) \leq 2 \smCE(p, \D^*)\]
\end{theorem}

This theorem lets us efficiently approximate the lower distance to calibration, up to a constant factor, by computing  the smooth calibration error. An efficient algorithm for computing the smooth calibration error is given by \cite{hu2024testing}.

We now address the question of how close the upper and lower distances are.
Assume that all we know about the predictor $p$ and distribution $\D^* = (\x, \y^*)$ is the distribution $J^* = (p(\x), \y^*)$. Does this specify $\dCE(p, D^*)$ completely? Or is there still some uncertainty about how far the closest calibrated predictor is, depending on the space $\X$? The answer (perhaps surprisingly) is that there is quadratic uncertainty in the distance, given $J^*$. 

\begin{corollary}
\label{cor:indist}
    No calibration measure based on $J^*$ can distinguish between the cases where $\dCE(p, \D^*) \geq \eta$ and $\dCE(p, \D^*) \leq 2\eta^2$. 
\end{corollary}

We present an example illustrating Corollary~\ref{cor:indist} in Appendix~\ref{app:example}. Specifically, we construct pairs of predictors and distributions $(p_1, \mD_1^*)$ and $(p_2, \mD_2^*)$ so that $J^*$ is identical in both cases, but $\dCE$ differs by a quadratic factor. It turns out that this quadratic separation is in fact the worst possible. 
\begin{theorem}
[\cite{BlasiokGHN23}]
\label{thm:int-upper}
    We have $\uCE(J^*) \leq 4\sqrt{\lCE(J^*)}$.   
\end{theorem}
We discuss the proof of this theorem in \Cref{sec:relate-appendix} following the original approach of \cite{BlasiokGHN23} via the notion of \emph{interval calibration error}.

\subsection{The inherent uncertainty in distance to calibration}
\label{app:example}
Assume that all we know about the predictor $p$ and distribution $\D^* = (\x, \y^*)$ is the distribution $J^* = (p(\x), \y^*)$. Does this specify $\dCE(p, D^*)$ completely? Or is there still some uncertainty on how far the closest calibrated predictor is, depending on the space $\X$?

We present a simple example showing that there is indeed some uncertainty.  Take $\eps$ to be any value in $(0,1/2)$, and let $\delta = \eps/(1 - 2\eps)$. The distribution $J^*$ is easy to describe: $p(\x)$ takes the values $1/2 + \delta$ and $1/2 - \delta$ each with probability $1/2$, and conditioned on each value of $p(\x)$, $\y^*$ is uniformly distributed in $\zo$.\eat{
\begin{center}
    \begin{tabular}{||c|c|c||} 
    \hline
         $v$ & $\Pr_{J^*}[p(\x) =v]$ & $\E_{J^*}[\y^*|v]$\\ [0.5ex] 
         \hline\hline
         $\fr{2} - \delta$ & $\fr{2}$ & $\fr{2}$  \\
         \hline
         $\fr{2} + \delta$ & $\fr{2}$ & $\fr{2}$  \\
         \hline
    \end{tabular}
\end{center}}

Note that any such $p$ is not perfectly calibrated. But it is $\delta$ far from the constant $1/2$ predictor, which is perfectly calibrated. It is easy to construct a space where this is indeed the closest calibrated predictor, so that  $\dCE(p, \mD^*) = \delta$. 

What is perhaps less obvious is there exist spaces and predictors realizing $J^*$ where the true distance to calibration is much smaller. We describe one such construction. Let $\X = \{ 00, 01, 10, 11\}$. Cosndier the distribution $\D^*$ on pairs $(\x, \y^*) \in \X \times \zo$, and predictors $p_1, p_2:\X \to [0,1]$ given below:
\begin{center}
    \begin{tabular}{||c|c|c|c|c||} 
    \hline
         $x$ & $\Pr_{\D^*}[\x = x]$ & $\E_{\D^*}[\y^*|\x  =x]$ & $p_1(x)$ & $p_2(x)$\\ [0.5ex] 
         \hline\hline
         $00$ & $\fr{2} -\eps$ & $\frac{1}{2} - \delta$ & $\fr{2} - \delta$ & $\fr{2} - \delta$  \\ 
         \hline
         $01$ & $\eps$ & $1$ & $\fr{2} - \delta$ & $\fr{2}$\\
         \hline
         $10$ & $\eps$ & $0$ & $\fr{2} + \delta$ & 
         $\fr{2}$ \\
         \hline
         $11$ & $\fr{2} - \eps$ &  $\frac{1}{2} + \delta$ & $\fr{2} + \delta$ & $\fr{2} + \delta$\\ [1ex]
         \hline
    \end{tabular}
\end{center}
\eat{
    \begin{align*}
    \Pr[\x = 00] &= \frac{1 - \eps}{2}, \ \E[\y^*|\x = 00] = \frac{1}{2} - \frac{\eps}{2(1 - \eps)},\\
    \Pr[\x = 01] &= \frac{\eps}{2}, \  \E[\y^*|\x =01] = 1,\\
    \Pr[\x = 10] &= \frac{\eps}{2}, \  \E[\y^*|\x = 10] = 0,\\ 
    \Pr[\x = 11] &= \frac{1- \eps}{2}, \ \E[\y^*|\x = 11] = \frac{1}{2} + \frac{\eps}{2(1 - \eps)}.
    \end{align*}
 }   
    
    The predictor $p_1$ is not perfectly calibrated, indeed  we have chosen $\delta$ such that the joint distribution of $(p_1(\x), \y^*)$ is exactly $J^*$: conditioned on either prediction value  in $\{1/2 \pm \delta\}$, the bit $\y^*$ is uniformly random.  In contrast, the predictor $p_2$ is easily seen to be calibrated.
    
    Observe that $p_1$ and $p_2$ agree on $00$ and $11$. They disagree by $\delta$ on $01$ and $10$, which each have $\eps$ probability under $\D^*$, so $d(p_1, p_2) = 2\eps\delta = \Theta(\eps^2)$. This establishes the difficulty of pinning down the true distance to calibration within a quadratic factor. 

\subsection{Relating Upper and Lower Distances to Calibration}
\label{sec:relate-appendix}
In this subsection, we prove \Cref{thm:int-upper} showing that the upper and lower distance to calibration can be at most quadratically far apart. This shows that the simple example in \Cref{app:example} is nearly tight.
We follow the proof strategy of \cite{BlasiokGHN23} using the notion of \emph{interval calibration error}.

\subsubsection{Interval Calibration Error}

\begin{definition}[Interval Calibration Error \cite{BlasiokGHN23}]
A interval partition $B$ is a partition of the interval $[0,1]$ into disjoint intervals $I_1, \ldots, I_k$. We let the width of the partition $\width(B)$ be the length of longest interval.
Given a predictor $p$, we define its  calibration error and interval calibration error for $B$ respectively as
\begin{align*}
\CE_B(p, \D^*) &= \sum_{j \in [k]}|\E[(\y^* - p(\x))\Ind(p(\x) \in I_j)]|\\
\intCE_B(p, \D^*) &= \CE_B(p, J^*) + \width(B).
\end{align*}
The \emph{interval calibration error}  minimizes over all interval partitions $B$:
\[ \intCE(p, \D^*) = \min_B \intCE_B(p, \D^*).\]
\end{definition}

The definition of $\intCE_B$ involves two terms that represent a tradeoff: the calibration error term, and the width term that penalizes partitions which use large width intervals. Intuitively, as the intervals grow larger it is easier to reduce calibration error, since we are allowed to cancel out the point-wise errors $\E[\y^*|p(\x)]  - p(\x)$ over larger intervals; but the width penalty also grows larger. At one extreme, we can think of the width $0$ case as corresponding to the $\ECE$. At the other extreme, by taking the single interval $[0,1]$, we pay $\E[\y^* - p(\x)]$ which is $0$ if the expectations of $\y^*$ and $p(\x)$ are equal; a very weak calibration guarantee. But now the width penalty is $1$. 

Formal justification for the definition comes from the following observation. The canonical predictor $q_B$ for an interval partition $B$ and a distribution $\D^*$ is the predictor where for all $x \in I_j$, the $q_B$ predicts  $v_j = \E[y^*|p(\x) \in I_j|$. It is easy to see that $q_B$ is perfectly calibrated for $\D^*$. 

\begin{lemma}
\label{lem:int-ce}
The canonical predictor $q_B$ for $B, \D^*$ satisfies $d(p, q_B) \leq \intCE_B(p, \D^*)$.
\end{lemma}
\begin{proof}
    Let $w_j = \E[p(\x)|p(\x) \in I_j]$, and note that $w_j \in I_j$, a property that will be used shortly. We can write
    \begin{align}
        \label{eq:v-w}
    \intCE_B(p, J^*) = \width(B) + \sum_j \Pr[p(\x) \in I_j]|v_j - w_j|. 
    \end{align}
    
     We now bound $d(p, q_B)$ as
    \begin{align*} 
    d(p, q_B) &= \E_{\D^*}[|p(\x) - q_B(x)|] \\
    & = \sum_{j \in [k]}\Pr[p(\x) \in I_j]\E[|p(\x) - v_j||p(\x) \in I_j] \\
    & \leq \sum_{j \in [k]} \Pr[p(\x) \in I_j]\left( \E[|p(\x) - w_j||p(\x) \in I_j] +  |w_j - v_j| \right)\\
    & \leq (\sum_{j \in [k]} \Pr[p(\x) \in I_j])\width(B) + \sum_{j \in [k]} \Pr[p(\x) \in I_j]|v_j - w_j|\\
    &= \intCE_B(p, J^*) \ \ \ \ (\text{By equation}  \eqref{eq:v-w})
    \end{align*}
    where the penultimate line uses the fact that conditioned on $p(\x) \in I_j$, $|p(\x) - w_j| \leq \width(B)$ since both values lie in the interval $I_j$.
\end{proof}

This leads to the following upper bound: 
\begin{theorem}
\cite{BlasiokGHN23}
\label{cor:intCE}
    We have $\uCE(p, \D^*) \leq \intCE(p, \D^*)$.
\end{theorem}

To prove Theorem \ref{cor:intCE} we observe that the canonical predictor $q_B$ can be viewed as a post-processing of the predictor $p$,  since we can write $q_B(x) = \kappa(p(x))$ where $\kappa(t) = v_j$ for $t \in I_j$.  Thus by Lemma~\ref{lem:int-ce}, 
\[
\uCE(p, \D^*) \le d(p, q_B) \leq \intCE_B(p, \D^*).
\]
Minimizing over all $B$ completes the proof. 

The reader might wonder, why define yet another calibration measure? The answer  is two-fold:
\begin{itemize}
    \item Interval calibration error gives a simple yet powerful upper bound on the upper distance to calibration. In the next subsection, this allows us to relate the upper and lower distance to calibration, showing that they are never more than quadratically far apart. This is formally proved in Theorem \ref{thm:int-upper}, showing the gap example in Corollary \ref{cor:indist} is the worst possible (up to constants).
    \item It presents a rigorous alternative to heuristic measures like bucketed $\ECE$: regularize the calibration error by adding the max bucket width. This allows for meaningful comparison of calibration scores obtained using different number or other choice of buckets, rather than leaving the number of buckets as a hyperparameter. 
\end{itemize}

\subsubsection{Proof of Theorem~\ref{thm:int-upper}}

Let us pick $\X, \D^*, p$ to be the space, distribution and predictor respectively that achieve the lower distance to calibration for $J^*$. So there exists a perfectly calibrated predictor $q:\X \to [0,1]$  such that $d(p, q) = \lCE(p, \D^*) = \delta$. We wish to infer the existence of a bucketing $B$ so that $\intCE_B(p, \D^*)$ is small. By Theorem \ref{cor:intCE}, this will imply that the upper distance is bounded. Corollary \ref{cor:indist} tells us that we cannot hope for an upper bound better than $\sqrt{\delta}/2$. It turns out that this is not far from the best possible:

\begin{lemma}
\label{lm:exist-bucket}
    There exists a bucketing $B$  such that $\intCE_B(p, \D^*) \leq 4\sqrt{\delta}$. 
\end{lemma}
Combining this lemma with \Cref{cor:intCE}, we have completed the proof:
\[
\uCE(p, \D^*) \leq \intCE(p, \D^*) \le \intCE_B(p, \D^*) \leq 4\sqrt{\delta} = 4\sqrt{\lCE(p,\D^*)}.
\]
\begin{proof}[Proof of \Cref{lm:exist-bucket}]
Let $\beta$ be a width parameter to be chosen later. We consider the bucketing $B$ where the first interval is $[0,b]$ for $b$ picked randomly from the interval $[0, \beta]$. Every subsequent interval has width $\beta$ (except possibly the last, which might be smaller). Denote the intervals by $I_1, \ldots, I_k$. 

For the predictor $q$, the calibration error term for $B$ is $0$ since
\[ \CE_B(q) = \sum_{j \in k}|\E[\Ind(q(\x) \in I_j)(\y^* - q(\x))]| \leq \int_{v \in [0,1]}\Pr[q(\x) = v]|\E[(\y^* - q(\x)|q(\x) = v]| = 0.\]

So we will try the bound the calibration term for $p$ by comparing it to $q$ and arguing that if they are close by, this error is small.
\begin{align}
  \CE_B(p, \D^*) &=   \sum_{j \in k}|\E[(\y^* - p(\x))\ind(p(\x) \in I_j)]|\notag\\ 
  & \leq \sum_{j \in k}|\E[(\y^* - q(\x))\ind(p(\x) \in I_j)]| + \sum_{j \in k} |\E[(q(\x) - p(\x))\ind(p(\x) \in I_j)]|\label{eq:tech0}
\end{align}
We bound each of these terms separately. To bound the second term,
\begin{align}
\sum_{j \in k} |\E[(q(\x) - p(\x))\ind(p(\x) \in I_j)]| = \E[|q(x) - p(x)|] \leq \delta \label{eq:tech1}
\end{align}
For the first term, we have
\begin{align*}
    \sum_{j \in k}|\E[(\y^* - q(\x))\ind(p(\x) \in I_j)]|  \leq &  \sum_{j \in k}|\E[(\y^* - q(\x))\ind(q(\x) \in I_j)]|  + \\
    & |\E[(\y^* - q(\x))(\ind(p(\x) \in I_j) - \ind(q(\x) \in I_j)]\\
    \leq & \sum_j|\ind(p(\x) \in I_j) - \ind(q(\x) \in I_j)|
\end{align*}
where we use $\CE_B(q,\D^*) =0$ and $|\y^* -q(\x)| \leq 1$. The RHS is $0$ if $p(\x)$ and $q(\x)$ land in the same bucket, else it is $2$. $p(\x)$ and $q(\x)$ land in different buckets if there is a bucket boundary between them, which happens with probability bounded by $|p(\x) - q(\x)|/\beta$ over the random choice of $b$. Hence we can bound 
\begin{align}
\label{eq:tech2}
    \sum_{j \in k}|\E[(\y^* - q(\x))\ind(p(\x) \in I_j)]| \leq  \frac{2\E[|p(\x) -q(\x)|]}{\beta} = \frac{2\delta}{\beta}.
\end{align}
Plugging Equations \eqref{eq:tech1} and \eqref{eq:tech2} back into Equation \eqref{eq:tech0} and choosing $\beta = \sqrt{2\delta}$,
\begin{align*}
\CE_B(p, \D^*) & \leq  \delta + 2\delta/\beta.\\
\intCE_B(p, \D^*) & \leq \CE_B(p, \D^*)  + \width(B)
\leq \beta + \delta + 2\sqrt{\delta} \leq 4\sqrt{\delta}.    \qedhere
\end{align*}
\end{proof}

\section*{Conclusion}
The classic notion of calibration needs to be rethought in order to satisfy requirements like robustness and computational efficiency, motivated by applications to machine learning and decision making. This leads to a rich set of new questions, in terms of what are desirable properties for approximate calibration notions to have and new algorithmic challenges that arise from trying to achieve these properties. This is a broad and active area of research that spans machine learning, decision making and computational complexity.  There are several questions that still remain, such as efficient and meaningful notions of calibration for the multiclass setting \cite{GopalanHR24} and the generative setting \cite{kalai-vempala}.
We hope to have given the reader a feel for this in the survey, by highlighting the motivating questions, the definitional challenges and the algorithmic issues.  

\bibliographystyle{alphaurl}
\bibliography{refs}

\newcommand{\etalchar}[1]{$^{#1}$}
\begin{thebibliography}{BGHN23b}

\bibitem[ACRS25]{elementary}
Eshwar~Ram Arunachaleswaran, Natalie Collina, Aaron Roth, and Mirah Shi.
\newblock An elementary predictor obtaining {$2\sqrt T + 1$} distance to
  calibration.
\newblock In Yossi Azar and Debmalya Panigrahi, editors, {\em Proceedings of
  the 2025 Annual {ACM-SIAM} Symposium on Discrete Algorithms, {SODA} 2025, New
  Orleans, LA, USA, January 12-15, 2025}, pages 1366--1370. {SIAM}, 2025.

\bibitem[BGHN23a]{BlasiokGHN23}
Jaroslaw Blasiok, Parikshit Gopalan, Lunjia Hu, and Preetum Nakkiran.
\newblock A unifying theory of distance from calibration.
\newblock In {\em Proceedings of the 55th Annual {ACM} Symposium on Theory of
  Computing, {STOC} 2023}, pages 1727--1740. {ACM}, 2023.

\bibitem[BGHN23b]{when-does}
Jaroslaw Blasiok, Parikshit Gopalan, Lunjia Hu, and Preetum Nakkiran.
\newblock When does optimizing a proper loss yield calibration?
\newblock In A.~Oh, T.~Naumann, A.~Globerson, K.~Saenko, M.~Hardt, and
  S.~Levine, editors, {\em Advances in Neural Information Processing Systems},
  volume~36, pages 72071--72095. Curran Associates, Inc., 2023.

\bibitem[BN24]{BlasiokN24}
Jaroslaw Blasiok and Preetum Nakkiran.
\newblock Smooth {ECE:} principled reliability diagrams via kernel smoothing.
\newblock In {\em The Twelfth International Conference on Learning
  Representations, {ICLR} 2024}, 2024.

\bibitem[CDV24]{CDV}
S\'{\i}lvia Casacuberta, Cynthia Dwork, and Salil Vadhan.
\newblock Complexity-theoretic implications of multicalibration.
\newblock In {\em Proceedings of the 56th Annual ACM Symposium on Theory of
  Computing}, STOC 2024, page 1071–1082, New York, NY, USA, 2024. Association
  for Computing Machinery.
\newblock \href {https://doi.org/10.1145/3618260.3649748}
  {\path{doi:10.1145/3618260.3649748}}.

\bibitem[DDF{\etalchar{+}}25]{breaking}
Yuval Dagan, Constantinos Daskalakis, Maxwell Fishelson, Noah Golowich, Robert
  Kleinberg, and Princewill Okoroafor.
\newblock Breaking the t{\^{}}(2/3) barrier for sequential calibration.
\newblock In Michal Kouck{\'{y}} and Nikhil Bansal, editors, {\em Proceedings
  of the 57th Annual {ACM} Symposium on Theory of Computing, {STOC} 2025,
  Prague, Czechia, June 23-27, 2025}, pages 2007--2018. {ACM}, 2025.

\bibitem[DKR{\etalchar{+}}21]{OI}
Cynthia Dwork, Michael~P. Kim, Omer Reingold, Guy~N. Rothblum, and Gal Yona.
\newblock Outcome indistinguishability.
\newblock In {\em ACM Symposium on Theory of Computing (STOC'21)}, 2021.
\newblock URL: \url{https://arxiv.org/abs/2011.13426}.

\bibitem[DLLT23]{DworkLLT}
Cynthia Dwork, Daniel Lee, Huijia Lin, and Pranay Tankala.
\newblock From pseudorandomness to multi-group fairness and back.
\newblock In Gergely Neu and Lorenzo Rosasco, editors, {\em Proceedings of
  Thirty Sixth Conference on Learning Theory}, volume 195 of {\em Proceedings
  of Machine Learning Research}, pages 3566--3614. PMLR, 12--15 Jul 2023.

\bibitem[FGMS25]{Fishelson}
Maxwell Fishelson, Noah Golowich, Mehryar Mohri, and Jon Schneider.
\newblock High-dimensional calibration from swap regret.
\newblock {\em arXiv preprint arXiv:2505.21460}, 2025.

\bibitem[FV98]{FosterV98}
Dean~P. Foster and Rakesh~V. Vohra.
\newblock Asymptotic calibration.
\newblock {\em Biometrika}, 85(2):379--390, 1998.

\bibitem[GH25]{GH-survey}
Parikshit Gopalan and Lunjia Hu.
\newblock Calibration through the lens of indistinguishability.
\newblock {\em ACM SIGecom Exchanges}, 23(1), July 2025.
\newblock URL: \url{https://www.sigecom.org/exchanges/volume_23/1/HU.pdf}.

\bibitem[GHR24]{GopalanHR24}
Parikshit Gopalan, Lunjia Hu, and Guy~N. Rothblum.
\newblock On computationally efficient multi-class calibration.
\newblock In {\em The Thirty Seventh Annual Conference on Learning Theory},
  volume 247 of {\em Proceedings of Machine Learning Research}, pages
  1983--2026. {PMLR}, 2024.

\bibitem[GKR{\etalchar{+}}22]{omni}
Parikshit Gopalan, Adam~Tauman Kalai, Omer Reingold, Vatsal Sharan, and Udi
  Wieder.
\newblock Omnipredictors.
\newblock In {\em Innovations in Theoretical Computer Science (ITCS'2022)},
  2022.
\newblock URL: \url{https://arxiv.org/abs/2109.05389}.

\bibitem[GKSZ22]{GopalanKSZ22}
Parikshit Gopalan, Michael~P. Kim, Mihir Singhal, and Shengjia Zhao.
\newblock Low-degree multicalibration.
\newblock In {\em Conference on Learning Theory, 2-5 July 2022, London, {UK}},
  volume 178 of {\em Proceedings of Machine Learning Research}, pages
  3193--3234. {PMLR}, 2022.

\bibitem[GR07]{gneiting}
Tilmann Gneiting and Adrian~E Raftery.
\newblock Strictly proper scoring rules, prediction, and estimation.
\newblock {\em Journal of the American Statistical Association},
  102(477):359--378, 2007.

\bibitem[HJTY24]{hu2024testing}
Lunjia Hu, Arun Jambulapati, Kevin Tian, and Chutong Yang.
\newblock Testing calibration in nearly-linear time.
\newblock In {\em The Thirty-eighth Annual Conference on Neural Information
  Processing Systems}, 2024.

\bibitem[HKRR18]{hkrr2018}
{\'{U}}rsula H{\'{e}}bert{-}Johnson, Michael~P. Kim, Omer Reingold, and Guy~N.
  Rothblum.
\newblock Multicalibration: Calibration for the (computationally-identifiable)
  masses.
\newblock In {\em Proceedings of the 35th International Conference on Machine
  Learning, {ICML}}, 2018.

\bibitem[HW24]{cdl}
Lunjia Hu and Yifan Wu.
\newblock Predict to minimize swap regret for all payoff-bounded tasks.
\newblock In {\em 2024 IEEE 65th Annual Symposium on Foundations of Computer
  Science (FOCS)}, pages 244--263, 2024.

\bibitem[HWY25]{smooth-decision}
Jason Hartline, Yifan Wu, and Yunran Yang.
\newblock {Smooth Calibration and Decision Making}.
\newblock In {\em 6th Symposium on Foundations of Responsible Computing (FORC
  2025)}, volume 329, pages 16:1--16:26, 2025.

\bibitem[KF08]{kakadeF08}
Sham Kakade and Dean Foster.
\newblock Deterministic calibration and {Nash} equilibrium.
\newblock {\em Journal of Computer and System Sciences}, 74(1):115–130, 2008.

\bibitem[KKG{\etalchar{+}}22]{kim2022universal}
Michael~P Kim, Christoph Kern, Shafi Goldwasser, Frauke Kreuter, and Omer
  Reingold.
\newblock Universal adaptability: Target-independent inference that competes
  with propensity scoring.
\newblock {\em Proceedings of the National Academy of Sciences}, 119(4), 2022.

\bibitem[KLST23]{KLST23}
Bobby Kleinberg, Renato~Paes Leme, Jon Schneider, and Yifeng Teng.
\newblock U-calibration: Forecasting for an unknown agent.
\newblock In Gergely Neu and Lorenzo Rosasco, editors, {\em Proceedings of
  Thirty Sixth Conference on Learning Theory}, volume 195 of {\em Proceedings
  of Machine Learning Research}, pages 5143--5145. PMLR, 12--15 Jul 2023.

\bibitem[KSJ18]{KSJ18}
Aviral Kumar, Sunita Sarawagi, and Ujjwal Jain.
\newblock Trainable calibration measures for neural networks from kernel mean
  embeddings.
\newblock In {\em Proceedings of the 35th International Conference on Machine
  Learning}, volume~80 of {\em Proceedings of Machine Learning Research}, pages
  2805--2814. PMLR, 2018.

\bibitem[KV24]{kalai-vempala}
Adam~Tauman Kalai and Santosh~S. Vempala.
\newblock Calibrated language models must hallucinate.
\newblock STOC 2024, New York, NY, USA, 2024. Association for Computing
  Machinery.

\bibitem[LHSW22]{optimization-scoring-rule}
Yingkai Li, Jason~D. Hartline, Liren Shan, and Yifan Wu.
\newblock Optimization of scoring rules.
\newblock In {\em Proceedings of the 23rd ACM Conference on Economics and
  Computation}, EC '22, page 988–989, New York, NY, USA, 2022. Association
  for Computing Machinery.
\newblock \href {https://doi.org/10.1145/3490486.3538338}
  {\path{doi:10.1145/3490486.3538338}}.

\bibitem[McC56]{mccarthy}
John McCarthy.
\newblock Measures of the value of information.
\newblock {\em Proceedings of the National Academy of Sciences},
  42(9):654--655, 1956.

\bibitem[OKK25]{optimal-omni}
Princewill Okoroafor, Robert Kleinberg, and Michael~P Kim.
\newblock Near-optimal algorithms for omniprediction.
\newblock In {\em 2025 IEEE 66th Annual Symposium on Foundations of Computer
  Science (FOCS)}, 2025.

\bibitem[Pen25]{peng}
Binghui Peng.
\newblock High dimensional online calibration in polynomial time.
\newblock {\em arXiv preprint arXiv:2504.09096}, 2025.

\bibitem[QV21]{sidestep}
Mingda Qiao and Gregory Valiant.
\newblock Stronger calibration lower bounds via sidestepping.
\newblock In {\em Proceedings of the 53rd Annual ACM SIGACT Symposium on Theory
  of Computing}, STOC 2021, page 456–466, New York, NY, USA, 2021.
  Association for Computing Machinery.

\bibitem[QZ24]{dist-online}
Mingda Qiao and Letian Zheng.
\newblock On the distance from calibration in sequential prediction.
\newblock In Shipra Agrawal and Aaron Roth, editors, {\em Proceedings of Thirty
  Seventh Conference on Learning Theory}, volume 247 of {\em Proceedings of
  Machine Learning Research}, pages 4307--4357. PMLR, 30 Jun--03 Jul 2024.

\bibitem[RSB{\etalchar{+}}25]{test-action}
Raphael Rossellini, Jake~A Soloff, Rina~Foygel Barber, Zhimei Ren, and Rebecca
  Willett.
\newblock Can a calibration metric be both testable and actionable?
\newblock {\em arXiv preprint arXiv:2502.19851}, 2025.

\bibitem[Sav71]{savage}
Leonard~J. Savage.
\newblock Elicitation of personal probabilities and expectations.
\newblock {\em Journal of the American Statistical Association},
  66(336):783--801, 1971.

\end{thebibliography}

\appendix

\end{document}